\def\year{2023}\relax
\documentclass[letterpaper]{article} 
\usepackage{aaai23}  
\usepackage{times}  
\usepackage{helvet}  
\usepackage{courier}  
\usepackage[hyphens]{url}  
\usepackage{graphicx} 
\usepackage{natbib}  
\usepackage{caption} 
\DeclareCaptionStyle{ruled}{labelfont=normalfont,labelsep=colon,strut=off} 
\usepackage{algorithm}

\usepackage{booktabs}       
\usepackage{amsfonts}       
\usepackage{upgreek}
\usepackage[english]{babel}

\usepackage{amsmath, amsthm,mathtools}
\usepackage{multirow}
\usepackage{layouts}
\usepackage{subfigure}
\usepackage{wrapfig}
\usepackage{arydshln}
\makeatletter
  \renewcommand*\env@matrix[1][*\c@MaxMatrixCols c]{%
    \hskip -\arraycolsep
    \let\@ifnextchar\new@ifnextchar
  \array{#1}}
\makeatother

\newcommand{\deff}{\mbox{$\stackrel{\rm def}{=}$}}
\newcommand{\name}{Multi-RoundSecAgg} 
\newcommand{\namespace}{\name{ }}
\usepackage[noend]{algpseudocode}
\usepackage{dsfont}

\usepackage{enumitem}

\usepackage{babel}
\usepackage{xcolor}

\usepackage{mathtools}

\DeclareMathOperator*{\argmin}{arg\,min}

\DeclareMathAlphabet{\mathbfsl}{OT1}{ppl}{b}{it}

\newtheorem{theorem}{Theorem}
\newtheorem*{theorem*}{Theorem}

\newtheorem*{lemma*}{Lemma}

\newtheorem{lemma}{Lemma}
\newtheorem{remark}{Remark}

\newtheorem{example}{Example}
\newtheorem{assumption}{Assumption}

\usepackage{newfloat}
\usepackage{listings}
\floatstyle{ruled}
\newfloat{listing}{tb}{lst}{}
\floatname{listing}{Listing}
\title{Securing Secure Aggregation: Mitigating Multi-Round Privacy Leakage in Federated Learning}
\author {
    Jinhyun So\thanks{\noindent Equal contribution 
    }\textsuperscript{\rm 1},
    Ramy E. Ali$^*$\textsuperscript{\rm 1},
    Ba\c{s}ak G\"uler\textsuperscript{\rm 2},
    Jiantao Jiao\textsuperscript{\rm 3},
    A. Salman Avestimehr\textsuperscript{\rm 1}
}
\affiliations {
    \textsuperscript{\rm 1} University of Southern California (USC) \\
    \textsuperscript{\rm 2} University of California, Riverside \\
    \textsuperscript{\rm 3} University of California, Berkeley\\
    jinhyun.so@samsung.com, ramy.ali@samsung.com, bguler@ece.ucr.edu, jiantao@eecs.berkeley.edu, avestime@usc.edu
}

\begin{document}

\maketitle

\begin{abstract}
Secure aggregation is a critical component in federated learning (FL), which enables the server to learn the aggregate model of the users without observing their local models. Conventionally, secure aggregation algorithms focus only on ensuring the privacy of individual users in a single training round. We contend that such designs can lead to significant privacy leakages over multiple training rounds, due to partial user participation at each round of FL. In fact, we show that the conventional random user selection strategies in FL lead to leaking users' individual models within number of rounds that is linear in the number of users. To address this challenge, we introduce a secure aggregation framework, Multi-RoundSecAgg, with multi-round privacy guarantees. In particular, we introduce a new metric to quantify the privacy guarantees of FL over multiple training rounds, and develop a structured user selection strategy that guarantees the long-term privacy of each user (over any number of training rounds). Our framework also carefully accounts for the fairness and the average number of participating users at each round. Our experiments on MNIST, CIFAR-10 and CIFAR-100 datasets in the IID and the non-IID settings demonstrate the performance improvement over the baselines in terms of privacy protection and test accuracy. 

\end{abstract}

\section{Introduction}

Federated learning (FL) enables collaborative training of learning models over the data collected and stored locally by multiple data-owners. The training in FL is typically coordinated by a central server who maintains a global model that is updated locally by the users. 
The local updates are then aggregated by the server to update the global model. 
Throughout the training, the users never share their data with the server, rather, they only share their local updates. However, as shown recently,  the local models may still reveal substantial information about the local datasets, and the private training data can be reconstructed from the local models through inference or inversion attacks \cite{fredrikson2015model,nasr2019comprehensive,zhu2020deep,geiping2020inverting}.
\begin{figure}
    \centering
    \includegraphics[scale=0.275]{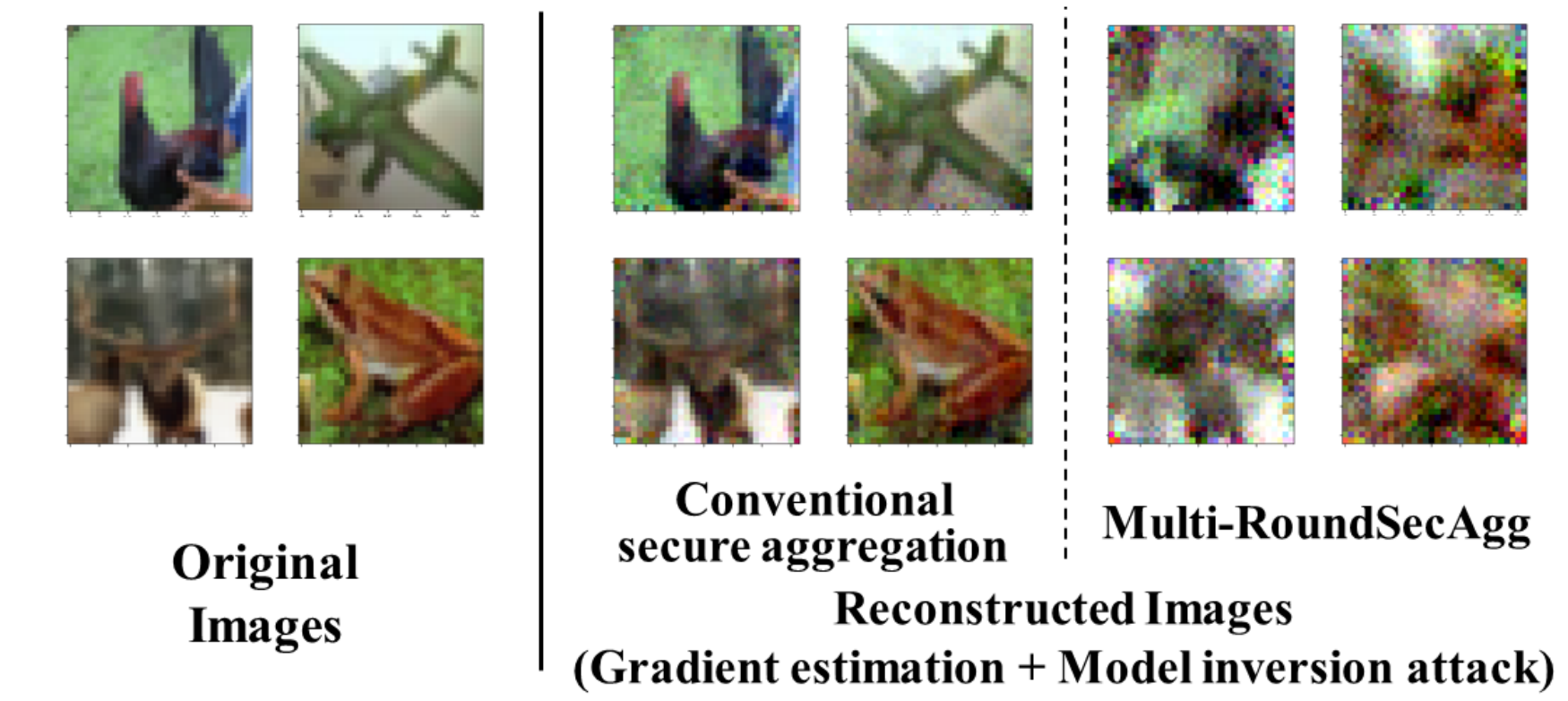}
    \caption{ 
    A qualitative comparison of the reconstructed images in two settings is shown.
    The first setting corresponds to the case that model privacy with random user selection (e.g., \texttt{FedAvg} \cite{mcmahan2017learning}) is protected by conventional secure aggregation schemes as \cite{bonawitz2017practical} at each round.
    In the second setting, our proposed method ensures the long-term privacy of individual models over any number of rounds, and hence model inversion attack cannot work well. 
    This reconstruction process is described in detail in Appendix \ref{app:random-selection}.} 
    

    \label{fig:intro_histo_error}
\end{figure}
\\To prevent such information leakage, \textit{secure aggregation} protocols are proposed (e.g., \cite{bonawitz2017practical,so2021turbo,kadhe2020fastsecagg,zhao2021information,bell2020secure, yang2021lightsecagg,so2021secure}) to protect the privacy of the local models, from the server and the other users, while still allowing the server to learn their aggregate. 
More specifically, the secure aggregation protocols ensure that, at any given round, the server can only learn the aggregate model of the users, and beyond that no further information is revealed about the individual model.

Secure aggregation, however, only ensures the privacy of the users in a \emph{single training round}, and do not consider their privacy over multiple training rounds \cite{bonawitz2017practical, bell2020secure, so2021turbo, so2022lightsecagg}. 
On the other hand, due to partial user selection  \cite{cho2020client,chen2020optimal,cho2020bandit,ribero2020communication}, the server may be able to reconstruct the individual models of some users using the aggregated models from the previous rounds. 
In fact, we show that after a sufficient number of rounds, all local models can be recovered with a high accuracy if the server uniformly chooses a random subset of the users to participate at every round. 
As shown in Fig.\ref{fig:intro_histo_error}, performing model inversion attack \cite{geiping2020inverting} with the recovered local models yields reconstructed images with a similar quality as the original images.


%
%
 \textbf{Contributions}. As such motivated, we study the long-term user privacy in FL.  Specifically, our contributions are as follows.
\begin{enumerate}[noitemsep,topsep=0pt,parsep=1pt,partopsep=2pt,leftmargin=*]
    \item 
    \color{black}
    We introduce a new metric to capture long-term privacy guarantees for secure aggregation protocols in FL for the first time. 
    \color{black}
    This long-term privacy requires that the server cannot reconstruct any individual model using the aggregated models from any number of training rounds. 
    Using this metric, we show that the conventional random selection schemes can result in leaking the local models after a sufficient number of rounds, even if secure aggregation is employed at each round.
    \item We propose \name, a privacy-preserving structured user selection strategy that ensures the long-term privacy of the individual users over any number of training rounds. This strategy also takes into account the fairness of the selection process and the average number of participating users  at each round. 
    \item We demonstrate that \namespace provides a trade-off between the long-term privacy guarantee and the average number of participating users. In particular, as the average number of participating users increases, the long-term privacy guarantee becomes weaker.
    
    \item We provide the convergence analysis of \name, which shows that the long-term privacy guarantee and the average number of participating users control the convergence rate. 
    The convergence rate is maximized when the average number of participating users is maximized 
    (e.g., the random user selection strategy maximizes the average number of participating users at the expense of not providing long-term privacy guarantees). 
    As we require stronger long-term privacy, the average number of participating users decreases and a larger number of training rounds is required to achieve the same level of accuracy as  random selection. 
    \item Finally, our experiments in both IID and non-IID settings on MNIST, CIFAR-$10$ and CIFAR-$100$  demonstrate that \namespace achieves almost the same test accuracy compared to the random selection scheme  while providing better long-term privacy guarantees.
\end{enumerate}



\section{Related Work}
\label{sec:RelatedWorks}
The underlying principle of the secure aggregation protocol in \cite{bonawitz2017practical} is that each pair of users exchange a pairwise secret key which they can use to mask their local models before sharing them with the server. 
The pairwise masks cancel out when the server aggregates the masked models, allowing the server to aggregate the local models. 
These masks also ensure that the local models are kept private, i.e., no further information is revealed beyond the aggregate of the local models. 
This protocol, however, incurs a significant communication cost due to exchanging and reconstructing the pairwise keys.

Several works also developed more efficient protocols \cite{so2021turbo,kadhe2020fastsecagg,bell2020secure,tang2021fedgp, choi2020communication, elkordy2020secure, yang2021lightsecagg}, which are complementary to and can be combined with our work. Another line of work focused on designing partial user selection strategies to overcome the communication bottleneck in FL while speeding up the convergence \cite{cho2020client,chen2020optimal,cho2020bandit,ribero2020communication}. 

Previous works, however, do not consider mitigating the potential privacy leakage as a result of partial user participation and the server observing the aggregated models across multiple training rounds. While \cite{pejo2020quality} pointed out to this problem, mitigating this leakage has not been considered and our work is the first \textcolor{black}{secure aggregation protocol} to address this challenge. Specifically, we identify a metric to quantify the long-term privacy of secure aggregation, and develop a privacy-preserving user selection strategy with provable long-term privacy.

\textcolor{black}{
Differential privacy (DP) techniques can also protect the privacy over the multiple FL rounds \cite{dwork2014algorithmic, abadi2016deep, wei2020federated, bonawitz2021federated, kairouz2021advances}, but this comes at the expense
of the model performance. It is worth noting that secure aggregation and DP are complementary, i.e., the benefits of DP can be applied to the secure aggregation protocols by adding noise to the local models \cite{bonawitz2021federated}.
In this paper, however, our objective is to understand the secure aggregation problem without DP. 
}

\section{System Model}
\label{sec:SystemModel}
We first describe the basic FL model in Section \ref{subsec:model}. Next, we introduce the multi-round secure aggregation problem for FL and define the key metrics to evaluate the performance of a multi-round secure aggregation protocol in Section \ref{subsec:multi-round_SA}.

\subsection{Basic Federated Learning Model}
\label{subsec:model}
We consider a cross-device FL setup consisting of a server and $N$ users. User $i\in[N]$ has a local dataset $\mathcal{D}_i$ consisting of $m_i=|\mathcal{D}_i|$ data samples. The users are connected to each other through the server, i.e., all communications between the users goes through the server \cite{mcmahan2016communication, bonawitz2017practical, kairouz2019advances}. The goal is to collaboratively learn a global model $\mathbfsl{x}$ with dimension $d$, using the local datasets that are generated, stored, and processed locally by the users. The training task can be represented by minimizing a global loss function,
\begin{equation}\label{eq:objective_fnc} 
    \min_{\mathbfsl{x}} L(\mathbfsl{x}) \text{ s.t.  } 
    L(\mathbfsl{x}) = \frac{1}{\sum_{i=1}^N w_i} \sum_{i=1}^N w_i L_i (\mathbfsl{x}), 
\end{equation} 
where $L_i$ is the loss function of user $i$ and $w_i \geq 0$ is a weight parameter assigned to user $i$ to specify the relative impact of that user.
A common choice for the weight parameters is $w_i=m_i$ \cite{kairouz2019advances}.
We define the optimal model parameters $\mathbfsl{x}^*$ and $\mathbfsl{x}_i^*$ as $\mathbfsl{x}^*=\argmin_{\mathbfsl{x}\in \mathbb{R}^d}L(\mathbfsl{x})$ and $\mathbfsl{x}_i^*=\argmin_{\mathbfsl{x}\in \mathbb{R}^d}L_i(\mathbfsl{x})$.

\noindent \textbf{Federated Averaging with Partial User Participation.}
To solve \eqref{eq:objective_fnc}, the most common algorithm is the \textit{FedAvg} algorithm  \cite{mcmahan2016communication}. \textit{FedAvg} is an iterative algorithm, where the model training is done by repeatedly iterating over individual local updates. 
At the beginning of training round $t$, the server sends the current global model $\mathbfsl{x}^{(t)}$ to the users. Each round consists of two phases, local training and aggregation.
In the local training phase, user $i\in[N]$ updates the global model by carrying out $E$ ($\geq1$) local stochastic gradient descent (SGD) steps and sends the updated local model $\mathbfsl{x}^{(t)}_i$ to the server. 
One of key features of cross-device FL is partial device participation. 
Due to various reasons such as unreliable wireless connectivity, at any given round, only a fraction of the users are available to participate in the protocol. We refer to such users as \emph{available} users throughout the paper. 
In the aggregation phase, the server selects $K\leq N$ users among the available users if this is possible and aggregates their \textcolor{black}{local updates}. The server updates the global model as follows
\begin{align}
\label{eq:aggregation}
    \mathbfsl{x}^{(t+1)} 
    = \sum_{i\in \mathcal{S}^{(t)}} w'_i\mathbfsl{x}^{(t)}_i
    = {\mathbf{X}^{(t)}}^\top \mathbfsl{p}^{(t)},
\end{align}
where $\mathcal S^{(t)}$ is the set of participating users at round $t$, $\mathbfsl{p}^{(t)}\in \{0,1\}^N$ is the corresponding characteristic vector and $w'_i=\frac{w_i}{\sum_{i\in \mathcal{S}^{(t)}} w_i}$. 
That is, $\mathbfsl{p}^{(t)}$ denotes a participation vector at round $t$ whose $i$-th entry is $0$ when user $i$ is not selected and $1$ otherwise. 
$\mathbf{X}^{(t)}$ denotes the concatenation of the weighted local models at round $t$, i.e., $\mathbf{X}^{(t)} = \big[ w'_1\mathbfsl{x}^{(t)}_1,\ldots,w'_N\mathbfsl{x}^{(t)}_N\big]^\top \in \mathbb{R}^{N \times d}$.
Finally, the server broadcasts the updated global model $\mathbfsl{x}^{(t+1)}$ to the users for the next round. \\
\color{black}
\noindent \textbf{Threat Model.}  Similar to the prior works on secure aggregation as \cite{bonawitz2017practical, kadhe2020fastsecagg, so2021turbo}, we consider the honest-but-curious model. All participants follow the protocol honestly in this model, but try to learn as much as possible about the users. 
At each round, the privacy of individual model $\mathbfsl{x}^{(t)}_i$ in \eqref{eq:aggregation} is protected by secure aggregation such that the server only learns the aggregated model $\sum_{i\in \mathcal{S}^{(t)}} w'_i\mathbfsl{x}^{(t)}_i$.
\color{black}
\subsection{Multi-round Secure Aggregation}
\label{subsec:multi-round_SA}
While secure aggregation protocols have provable privacy guarantees at any single round, in the sense that no information is leaked beyond the aggregate model at each round, the privacy guarantees do not extend to attacks \emph{that span multiple training rounds}. 
Specifically, by using the aggregate models and participation information across multiple rounds, an individual model may be reconstructed. For instance, consider the following user participation strategy across three training rounds,  $\mathbfsl{p}^{(1)} = [1, 1, 0]^\top$, $\mathbfsl{p}^{(2)} = [0, 1, 1]^\top$, and $\mathbfsl{p}^{(3)} = [1, 0, 1]^\top$. 
Assume a scenario where the local updates do not change significantly over time (e.g., models start to converge), i.e., $\mathbfsl{x}_i=\mathbfsl{x}^{(t)}_i$ for all $i \in [3]$ and $t \in [3]$. Then the server can single out 
all individual models, even if a secure aggregation protocol is employed at each  round.

In this paper, we study secure aggregation protocols with long-term privacy guarantees (which we term \emph{multi-round secure aggregation}) for the cross-device FL setup. 
We assume that user $i\in[N]$ drops from the protocol at each round with probability $p_i$. 
$\mathcal{U}^{(t)}$ denotes the index set of available users at round $t$ and $\mathbfsl{u}^{(t)}\in\{0,1\}^N$ is a vector indicating the available users such that $\{\mathbfsl{u}^{(t)}\}_{j} = \mathds{1}\{ j \in \mathcal{U}^{(t)} \}$, where $\{\mathbfsl{u}\}_j$ is $j$-th entry of $\mathbfsl{u}$ and $\mathds{1}\{ \cdot \}$ is the indicator function. The server selects $K$ users from $\mathcal{U}^{(t)}$, if  $|\mathcal{U}^{(t)}|\geq K$, based on the history of selected users in previous rounds. If $|\mathcal{U}^{(t)}|<K$, the server skips this round. The local models of the selected users are then aggregated via a secure aggregation protocol (i.e., by communicating masked models),
 at the end of which the server learns the aggregate of the local models of the selected users. Our goal is to design a user selection algorithm  $\mathcal{A}^{(t)}:\{0,1\}^{t\times N}\times \{0,1\}^{N} \rightarrow \{0,1\}^N$, 
\begin{equation}
\label{eq:def_selection}
    \mathcal{A}^{(t)}\big( \mathbf{P}^{(t)}, \mathbfsl{u}^{(t)} \big) = \mathbfsl{p}^{(t)} 
    \text{ such that } \| \mathbfsl{p}^{(t)} \|_0 \in \{0, K\},
\end{equation}
to prevent the potential information leakage over multiple rounds, where $\mathbfsl{p}^{(t)}\in \{0,1\}^N$ is the participation vector defined in \eqref{eq:aggregation}, $\| \mathbfsl x \|_0$ denotes the $L_0$-``norm'' of $\mathbfsl{x}$ and $K$ denotes the number of selected users. 
We note that $\mathcal{A}^{(t)}$ can be a random function.
$\mathbf{P}^{(t)}$ is a matrix representing the user participation information up to round $t$, and is termed the \emph{participation matrix}, given by
\begin{equation} \label{eq:def_P}
    \mathbf{P}^{(t)} = \big[ \mathbfsl{p}^{(0)}, \mathbfsl{p}^{(1)}, \ldots, \mathbfsl{p}^{(t-1)}\big]^{\top} \in \{0,1\}^{t\times N}.
\end{equation}
\noindent \textbf{Key Metrics.}
A multi-round secure aggregation protocol can be represented by $\mathcal{A}=\{ \mathcal{A}^{(t)}\}_{t\in[J]}$, where $\mathcal{A}^{(t)}$ is the user selection algorithm at round $t$ defined in \eqref{eq:def_selection} and $J$ is the total number of rounds.  The  inputs of  $\mathcal{A}^{(t)}$ are a random vector $\mathbfsl{u}^{(t)}$, which indicates the available users at round $t$, and the participation matrix $\mathbf{P}^{(t)}$ defined in \eqref{eq:def_P} which can be a random matrix.
Given the participation matrix $\mathbf{P}^{(J)}$, we evaluate the performance of  the corresponding multi-round secure aggregation protocol through the following metrics.

\begin{enumerate}
    \item \textbf{Multi-round Privacy Guarantee.}  Secure aggregation protocols ensure that the server can only learn the sum of the local models of some users in each single round, but they do not consider what the server can learn over the long run. Our multi-round privacy definition extends the guarantees of the secure aggregation protocols from one round to all rounds by requiring that the server can only learn a sum of the local models even if the server exploits the aggregate models of all rounds. That is, our multi-round privacy guarantee is a natural extension of the privacy guarantee provided by the secure aggregation protocols considering a single training round.

    Specifically, a multi-round privacy guarantee $T$ requires that any non-zero partial sum of the local models that the server can reconstruct, through any linear combination $\mathbf X^\top {\mathbf P^{(J)}}^\top  \mathbfsl z$, where $\mathbfsl z \in \mathbb R^{J} \setminus \{\mathbf 0\}$, must be of the form\footnote{We assume that $w_i=\frac{1}{N}, \forall i \in [N]$ in this paper.}
    \begin{align} 
    \label{eq:def_privacy}
      & \mathbf X^\top {\mathbf P^{(J)}}^\top  \mathbfsl z =\sum_{i \in [n]}a_i \sum_{j \in \mathcal S_i} \mathbfsl x_j \notag \\&=  a_1 \sum_{j \in \mathcal S_1} \mathbfsl x_j+a_2 \sum_{j \in \mathcal S_2} \mathbfsl x_j+ \cdots+a_n \sum_{j \in \mathcal S_n} \mathbfsl x_j,
    \end{align}
where $a_i\neq 0, \forall i \in [n],n \in \mathbb Z^+, |\mathcal S_i| \geq T$ and $\mathcal S_i \cap \mathcal S_j = \emptyset$ when $i \neq j$. Here all the sets $\mathcal{S}_i$, the number of sets $n$, and each $a_i$ could all depend on $\mathbfsl{z}$. In equation (\ref{eq:def_privacy}), we consider the worst-case scenario, where the local models do not change over the rounds. That is, $\mathbf{X}^{(t)}=\mathbf{X}, \ \forall t\in [J]$. Intuitively, this guarantee ensures that the best that the server can do is to reconstruct a partial sum of $T$ local models which corresponds to the case where $n=1$. 
When $T\geq2$, this condition implies that the server cannot get any user model from the aggregate models of all training rounds (the best it can obtain is the sum of two local models).
\begin{remark}\label{remark:weaker_T} \normalfont (Weaker Privacy Notion). It is worth noting that, a weaker privacy notion would require that $\| {\mathbf{P}^{(J)}}^\top \mathbfsl{z}\|_0 \geq T$ when ${\mathbf{P}^{(J)}}^\top \mathbfsl{z} \neq \mathbf 0$. When $T=2$, this definition requires that the server cannot reconstruct any individual model (the best it can do is to obtain a linear combination of two local models). This notion, however, allows constructions in the form of $a \mathbfsl x_i+b \mathbfsl x_j$ for any $a \neq 0, b \neq 0$.
When $a \gg b$, however, this is almost the same as  recovering $\mathbfsl x_i$ perfectly, hence this privacy criterion is weaker than that of \eqref{eq:def_privacy}. We refer to \cite{august2022} for a follow-up work that considers this weaker notion.
\end{remark}
\begin{remark}\label{remark:random-selection-privacy} \normalfont(Multi-round Privacy of Random Selection). In Section \ref{sec:Experiments}, we empirically show that a random selection strategy in which $K$ available users are selected uniformly at random at each round does not ensure multi-round privacy even with respect to the weaker definition of Remark \ref{remark:weaker_T}. Specifically, the  local models can be reconstructed within a number of rounds that is linear in $N$. We also show theoretically in Appendix \ref{app:random-selection}
that when $\min(N-K, K) \geq cN$, where $c>0$ is a constant, then the probability that the server can reconstruct all local models after $N$ rounds is at least $1-2e^{-c'N}$ for a constant $c'$ that depends on $c$. Finally, we show that a random selection scheme in which the users are selected in an i.i.d fashion according to Bern($\frac{K}{N(1-p)}$) reveals all local models after $N$ rounds with probability that converges to $1$ exponentially fast. 
\end{remark}
\begin{remark}\label{remark:model-change-time}\normalfont
(Worst-Case Assumption). In (\ref{eq:def_privacy}), we considered the worst-case assumption where the models do not change over time. When the models change over rounds, the multi-round privacy guarantee becomes even stronger as the number of unknowns increases. 
In Fig. \ref{fig:intro_histo_error} and Appendix \ref{app:random-selection}, 
we empirically show that the  conventional secure aggregation schemes leak extensive information of training data even in the realistic settings where the models change over the rounds.
\end{remark}

    \item \textbf{Aggregation Fairness Gap.}
    The average aggregation fairness gap quantifies the largest gap between any two users in terms of the expected relative number of rounds each user has participated in training. Formally, the average aggregation fairness gap is defined as follows

    \begin{align}
    \label{eq:def_fairness}
        F &= \max_{i \in [N]} \limsup\limits_{J\rightarrow \infty} \frac{1}{J}\mathbb{E} \Big[ \sum_{t=0}^{J-1} \mathds{1}
        \big\{ \{\mathbfsl{p}^{(t)}\}_i = 1 \big\} \Big]   - \notag \\& \min_{i \in [N]} \liminf\limits_{J\rightarrow \infty} \frac{1}{J}\mathbb{E} \Big[ \sum_{t=0}^{J-1} \mathds{1}
        \big\{ \{\mathbfsl{p}^{(t)}\}_i = 1 \big\} \Big],
    \end{align}
    where  $\{\mathbfsl{p}^{(t)}\}_i$ is $i$-th entry of  the vector $\mathbfsl{p}^{(t)}$ and the expectation is over the randomness of the user selection algorithm $\mathcal{A}$ and the user availability. The main intuition behind this definition is that when $F=0$, all users participate on average on the same number of rounds. This is important to take the different users into consideration equally and our experiments show that the accuracy of the schemes with small $F$ are much higher than the schemes with high $F$.

    \item \textbf{Average Aggregation Cardinality.} The aggregation cardinality quantifies the expected number of models to be aggregated per round. Formally, it is defined as
    \begin{equation}\label{eq:def_cardinality} 
        C = \liminf\limits_{J\rightarrow \infty} \mathbb{E}\big[\sum_{t=0}^{J-1} \| \mathbfsl{p}^{(t)}\|_0 \big]/J,
    \end{equation}
    where the expectation is over the randomness in $\mathcal{A}$ and the user availability.  Intuitively, less number of rounds are needed to converge as more users participate in the training. In fact,  as we show in Section \ref{subsec:convergence}, $C$ directly controls the convergence rate. 
\end{enumerate}

\subsection{Baseline Schemes}
\label{subsec:baselines}
In this subsection, we introduce three baseline schemes for multi-round secure aggregation. \\
\noindent \textbf{Random Selection.} In this scheme, at each round, the server selects $K$ users at random from the set of available users if this is possible.  \\
\noindent \textbf{Random Weighted Selection.} This scheme is a modified version of random selection to reduce $F$ when the dropout probabilities of the users are not equal. Specifically, $K$ users are selected at random from the available users with the minimum frequency of participation in the previous rounds. 
\\
\noindent \textbf{User Partitioning (Grouping).} 
In this scheme, the users are partitioned into $G=N/K$ equal-sized groups denoted as $\mathcal G_1, \cdots, \mathcal G_G$. At each round, the server selects one of the groups if none of the users in this group has dropped out. 
If multiple groups are available, to reduce $F$, the server selects a group including a user with the minimum frequency of participation in previous rounds. If no group is available, the server skips this round.
\section{Proposed Scheme: \name}
\label{sec:Proposed}
In this section, we present \name, which has two components as follows.  
\begin{itemize}[leftmargin=*]
    
    \item The first component designs a family of sets of users that satisfy the multi-round privacy requirement. 
    The inputs of the first component are the number of users ($N$), the number of selected \textcolor{black}{users} at each round ($K$), and the desired multi-round privacy guarantee ($T$). 
    The output is a family of sets of $K$ users  satisfying the multi-round privacy guarantee $T$, termed as a \emph{privacy-preserving family}. This family is represented by a matrix $\mathbf B$, where the rows are the characteristic vectors of these user sets. 
    
    %
    \item The second component selects a set from this  family to satisfy the fairness guarantee. 
    The inputs to this component are the privacy-preserving family represented by the matrix $\mathbf B$, the set of available users at round $t$, $\mathcal U^{(t)}$, and the frequency of participation of each user.
    The output is the set of users that will participate at round $t$. 
\end{itemize}
We now describe these two components in detail.

\noindent{\bf Component 1 (Batch Partitioning (BP) of the users to guarantee multi-round privacy).}
The first component designs a family of $R_{\textrm{BP}}$ sets, where $R_{\textrm{BP}}$ is the size of the set, satisfying the multi-round privacy requirement $T$. We denote the $R_{\textrm{BP}} \times N$ binary matrix corresponding to these sets by $\mathbf B=[\mathbfsl b_1, \cdots, \mathbfsl b_{R_{\textrm{BP}} }]^\top$, where $\| \mathbfsl b_i\|_0=K, \forall i\in[R_{\textrm{BP}}]$.
That is, the rows of $\mathbf B$ are the characteristic vectors of those sets. The main idea of our  scheme is to restrict certain sets of users of size $T$, denoted as batches, to either participate together or not participate at all. This guarantees a multi-round privacy $T$ as we show in Section \ref{sec:TheoreticalResults}. 

To construct a family of sets with this property, the users are first partitioned into $N/T$ batches.
At any given round, either all or none of the users of a particular batch participate in training. The server can choose $K/T$ batches to participate in training, provided that all users in any given selected batch are available. 
Since there are $\binom{N/T}{K/T}$ possible sets with this property, then the size of this privacy-preserving family of sets  is given by $R_{\textrm{BP}} \deff \binom{N/T}{K/T}$\footnote{We assume for simplicity that $N/T$ and $K/T$ are integers.}.
%
%
\begin{figure}
    \centering
    \includegraphics[scale=0.2]{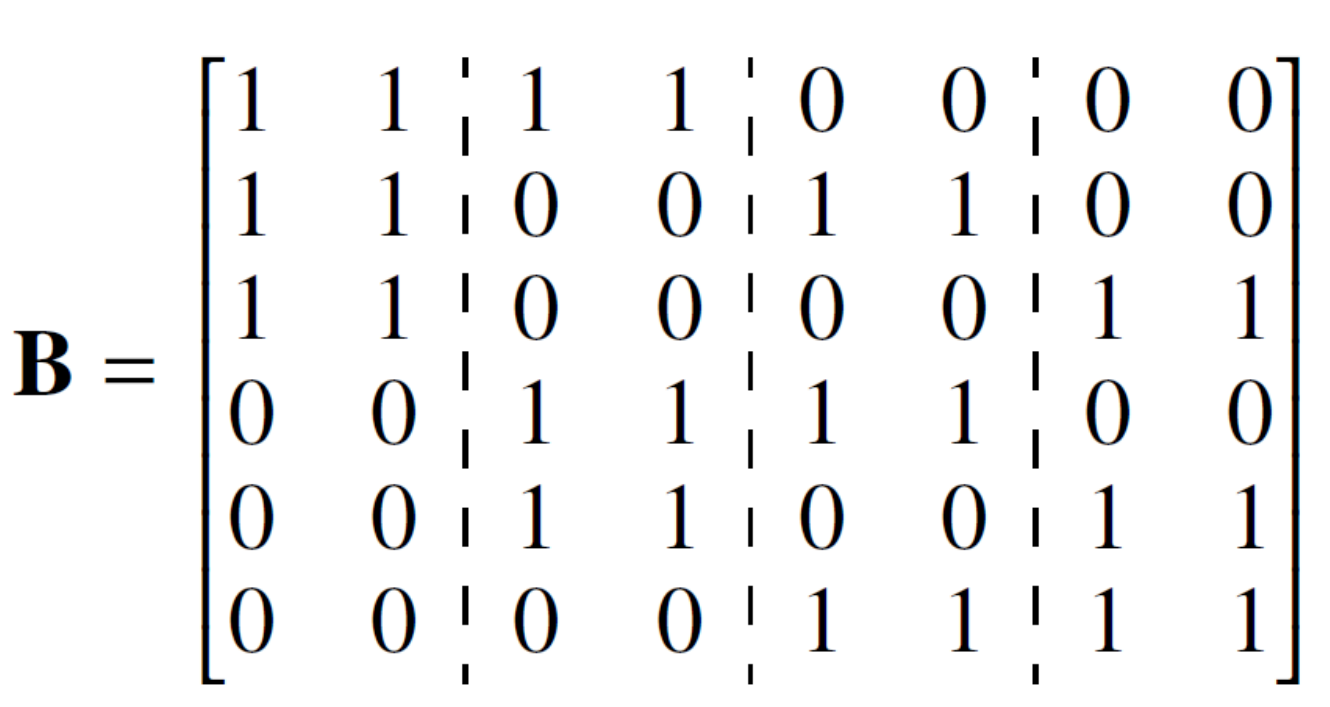}
    \caption{ Our construction with 
    $N=8$, $K=4$ and $T=2$.
    }
    
    \label{fig:codebook_example}
\end{figure}
In the extreme case of $T=1$, this strategy specializes to  random selection where the server can choose any $K$ possible users.
In the other extreme case of $T=K$, this strategy specializes to the partitioning strategy where there are $N/K$ possible sets. 
We next provide an example to illustrate the construction of $\mathbf B$ as shown in Fig. \ref{fig:codebook_example}.

\begin{example}[$N=8, K=4, T=2$]
In this example, the users are partitioned into $4$ batches as  $\mathcal G_1=\{1, 2\}, \mathcal G_2=\{3, 4\}, \mathcal G_3=\{5, 6\}$ and $\mathcal G_4=\{7, 8\}$ as given in Fig.~\ref{fig:codebook_example}. The server can choose any two batches out of these $4$ batches, hence we have $R_{\textrm{BP}}=\binom{4}{2}=6$ possible sets. This ensures a multi-round privacy $T=2$.

\end{example}


\noindent \textbf{Component 2 (Available batch selection to guarantee fairness)}.
At round $t$, user $i\in[N]$ is available to participate in the protocol with a probability $1-p_i \in (0,1]$. The frequency of participation of user $i$ before round $t$ is denoted by $f^{(t-1)}_i \deff \sum_{j=0}^{t-1}\mathds{1}\left\{ \{ \mathbfsl{p}^{(j)}\}_i=1 \right\}$.
Given the set of available users at round $t$, $\mathcal{U}^{(t)}$, and the frequencies of participation $\mathbfsl f^{(t-1)}=(f^{(t-1)}_1, \cdots, f^{(t-1)}_N)$, the server selects $K$ users.
To do so, the server first finds the submatrix of $\mathbf B$ denoted by $\mathbf B^{(t)}$ corresponding to $\mathcal{U}^{(t)}$. Specifically, the $i$-th row of $\mathbf B$ denoted by $\mathbfsl b_i^{\top}$ is included in $\mathbf B^{(t)}$ provided that $\mathrm{supp(\mathbfsl b_i)} \subseteq \mathcal{U}^{(t)}$. 
If $\mathbf B^{(t)}$ is an empty matrix, then the server skips this round. Otherwise, the server selects a row from $\mathbf B^{(t)}$ uniformly at random if $p_i=p, \forall i \in [N]$. If the users have different $p_i$, the server selects a row from $\mathbf B^{(t)}$ that includes the user with the minimum frequency of participation $\ell_{\mathrm{min}}^{(t-1)}\deff \argmin_{i \in \mathcal U^{(t)}} f^{(t-1)}_i$. If there are  many such rows, then the server selects one of them uniformly at random. 
\begin{remark}\normalfont (Necessity of the Second Component).
The second component is necessary to  guarantee that the aggregation fairness gap goes to zero as we show in Theorem \ref{thm:multi-roundsecagg-guarantees} and Section  \ref{sec:Experiments}.
\end{remark}
\noindent Overall, the algorithm  designs a privacy-preserving family of sets to ensure a multi-round privacy guarantee $T$. 
Then specific sets are selected from this family to ensure fairness. 
We describe the two components of \namespace in detail in Algorithms 1 
and 2 in Appendix \ref{app:algorithm}.
\section{Theoretical Results}
\label{sec:TheoreticalResults}
In this section, we provide the theoretical guarantees of \name.
\subsection{Guarantees of \name}
\label{subsec:multi-roundsecagg-guarantees}
In this subsection, we establish the theoretical guarantees of \namespace in terms of the multi-round privacy guarantee, the aggregation fairness gap and the average aggregation cardinality.
\begin{theorem}
\label{thm:multi-roundsecagg-guarantees}
\namespace with parameters $N, K, T$ ensures a multi-round privacy guarantee of $T$, an aggregation fairness gap $F=0$, and an average aggregation cardinality that is given by 
    \begin{align}
    {
        C=K\left( 1-\sum\limits_{i=N/T-K/T+1}^{N/T} \binom{N/T}{i} q^i (1-q)^{N/T-i} \right)}, \notag
    \end{align}
    where $q=1-(1-p)^T$, when all users have the dropout probability $p$. 
\end{theorem}
\noindent We provide the proof in Appendix \ref{app:privacy_proof}.  





\begin{figure}
    \centering
    \includegraphics[scale=0.13]{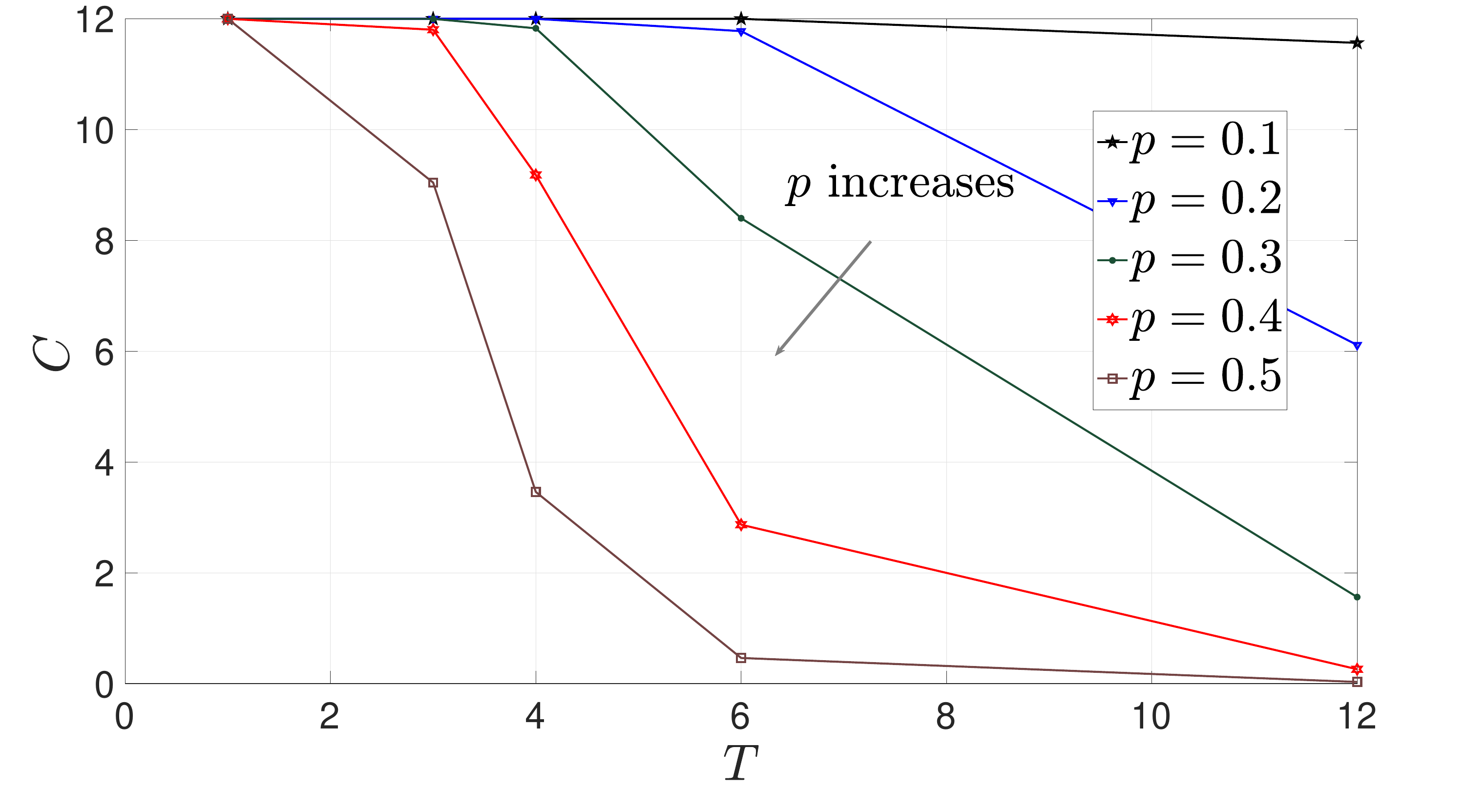}
    \caption{An illustration of the trade-off between the multi-round privacy guarantee $T$ and the average aggregation cardinality $C$. In this example, $N=120$ and $K=12$.
    }
    \label{fig:trade-off-anayltically}
\end{figure}
\begin{remark}\normalfont(Trade-off between ``Multi-round Privacy Guarantee'' and ``Average Aggregation Cardinality''). 
Theorem \ref{thm:multi-roundsecagg-guarantees} indicates a  trade-off between the multi-round privacy and the average aggregation cardinality since as $T$ increases, $C$ decreases which slows down the convergence as we show in Sec.  \ref{subsec:convergence}. We show this trade-off in Fig. ~\ref{fig:trade-off-anayltically}. 
\end{remark} 

\begin{remark}\normalfont(Necessity of Batch Partitioning (BP)).
\label{remark:BP_Necessity}
We show that any strategy that satisfies the privacy guarantee in Equation (\ref{eq:def_privacy}) must have a batch partitioning structure, and for given $N,K,T,K\leq N/2$, the largest number of distinct user sets in any strategy is at most ${N/T \choose K/T}$, which is achieved in our design in Section~\ref{sec:Proposed}. We provide the proof in Appendix~\ref{app:converse}.

\end{remark}
\begin{remark}\normalfont
(Non-linear Reconstructions of Aggregated Models). The privacy criterion in Eq. (\ref{eq:def_privacy}) considers linear reconstructions of the aggregated models. One may also consider more general non-linear reconstructions. The long-term privacy guarantees of batch partitioning hold even under such reconstructions as the users in the same batch always participate together or do not participate at all. Hence, the server cannot separate individual models within the same batch even through non-linear operations. 
\end{remark}



\subsection{Convergence Analysis of \namespace}
\label{subsec:convergence}
For convergence analysis, we first introduce a few common assumptions \cite{li2019convergence, yu2019parallel}.

\begin{assumption} \label{assumpt:1}
$L_1,\ldots,L_N$ in \eqref{eq:objective_fnc} are all $\rho$-smooth:  $\forall \mathbfsl{a},\mathbfsl{b}\in\mathbb{R}^d$ and $i\in[N]$, $L_i(\mathbfsl{a}) \leq L_i(\mathbfsl{b}) + (\mathbfsl{a}-\mathbfsl{b})^\top \nabla L_i(\mathbfsl{b}) + \frac{\rho}{2}\lVert \mathbfsl{a}-\mathbfsl{b} \rVert^2$.
\end{assumption}

\begin{assumption}\label{assumpt:2}
$L_1,\ldots,L_N$ in \eqref{eq:objective_fnc} are all $\mu$-strongly convex:  $\forall \mathbfsl{a},\mathbfsl{b}\in\mathbb{R}^d$ and $i\in[N]$, $L_i(\mathbfsl{a}) \geq L_i(\mathbfsl{b}) + (\mathbfsl{a}-\mathbfsl{b})^\top \nabla L_i(\mathbfsl{b}) + \frac{\mu}{2}\lVert \mathbfsl{a}-\mathbfsl{b} \rVert^2$.
\end{assumption}

\begin{assumption}\label{assumpt:3}
Let $\xi_i^{(t)}$ be a sample uniformly selected from the dataset $\mathcal{D}_i$. The variance of the stochastic gradients at each user is bounded, i.e., $\mathbb{E}\lVert \nabla L_i(\mathbfsl{x}_i^{(t)},\xi_i^{(t)}) - \nabla L_i(\mathbfsl{x}_i^{(t)}) \rVert^2 \leq \sigma^2_i$ for $i\in [N]$.
\end{assumption}

\begin{assumption}\label{assumpt:4}
The expected squared norm of the stochastic gradients is uniformly bounded, i.e., \\ $\mathbb{E}\lVert \nabla L_i(\mathbfsl{x}_i^{(t)},\xi_i^{(t)}) \rVert^2 \leq G^2$ for all $i\in [N]$.
\end{assumption}
\noindent We now state our convergence guarantees.

\begin{theorem} \label{thm:convergence} 
    Consider a FL setup with $N$ users to train a machine learning model from \eqref{eq:objective_fnc}. 
    Assume $K$ users are selected by {\name} with average aggregation cardinality $C$ defined in \eqref{eq:def_cardinality} to update the global model from \eqref{eq:aggregation}, and all users have the same dropout rate, hence {\name} selects a random set of $K$ users uniformly from the set of available user sets at each round.
    Then, 
    \begin{align}
    \label{eq:conv_rate}
        &\mathbb{E}[L(\mathbfsl{x}^{(J)})]-L^{*} \notag \\&\leq \frac{\rho}{\gamma + \frac{C}{K}EJ-1}
        \left( \frac{2(\alpha + \beta)}{\mu^2} + \frac{\gamma}{2} \mathbb{E} \lVert \mathbfsl{x}^{(0)} - \mathbfsl{x}^{*}\rVert ^2 \right),
    \end{align}
    where $\alpha = \frac{1}{N}\sum_{i=1}^{N}\sigma_i^2 + 6\rho\Gamma + 8(E-1)^2G^2$, $\beta=\frac{4(N-K)E^2G^2}{K(N-1)}$, $\Gamma=L^*-\sum_{i=1}^{N}L^*_i$, and $\gamma=\max\left\{\frac{8\rho}{\mu},E\right\}$.
\end{theorem}
\noindent We provide the proof in Appendix \ref{app:convergence_proof}.
\begin{remark} \normalfont (The average aggregation cardinality controls the convergence rate.)
    Theorem \ref{thm:convergence} shows how the average aggregation cardinality affects the convergence.  When the average aggregation cardinality is maximized, i.e., $C=K$, the  convergence rate in Theorem \ref{thm:convergence} equals that of the random selection algorithm provided in Theorem 3 of \cite{li2019convergence}. 
    In \eqref{eq:conv_rate}, we have the additional term $E$ (number of local epochs) in front of $J$ compared to Theorem 3 of \cite{li2019convergence} as we use global round index $t$ instead of using step index of local SGD.
    As the average aggregation cardinality decreases, a greater number of training rounds is required to achieve the same level of accuracy. 
\end{remark}
\begin{remark}\normalfont(General Convex and Non-Convex). Theorem \ref{thm:convergence} considers the strongly-convex case, but the general convex and the non-convex cases can be addressed as in \cite{karimireddy2020scaffold}.
\end{remark}
\begin{remark}\normalfont(Different Dropout Rates). When the dropout probabilities of the users are not the same, characterizing the convergence guarantees is challenging. This is due to the fact that batch selection based on the frequency of participation breaks the conditional unbiasedness of the user selection, which is required for the convergence guarantee. However, we empirically show that \namespace guarantees the convergence with different dropout rates.
\end{remark}

\section{Experiments}\label{sec:Experiments}

We first numerically demonstrate the performance of \namespace compared to the baselines of Sec. \ref{subsec:baselines} in terms of the key metrics of Sec. \ref{subsec:multi-round_SA}.
Next, we implement convolutional neural networks (CNNs) with MNIST~\cite{lecun2010mnist}, CIFAR-10, and CIFAR-100~\cite{krizhevsky2009learning} to investigate the effect of the key metrics on the test accuracy.
\begin{table}[htb!]
    \centering
    \scalebox{0.9}{
    \begin{tabular}[b]{cl} 
        \toprule    
        Scheme & Family size ($=R$) \\
        \midrule    
        Random selection  &  $\sim 10^{16}$ \\
        Weighted random selection &  $\sim 10^{16}$ \\
        User partition            &  $10$ \\
        {\name}, T=6    & 190 \\
        {\name}, T=4    & 4060 \\
        {\name}, T=3    & 91389 \\
        \bottomrule 
    \end{tabular}
    }
     \caption{Family size with $N=120$, $K=12$.
    \label{tbl:codebook_size}
    }
\end{table}

\begin{figure*}[ht!]
\centering
    \subfigure[Multi-round privacy guarantee.]{\label{fig:priavcy}
    \includegraphics[width=.32\textwidth]{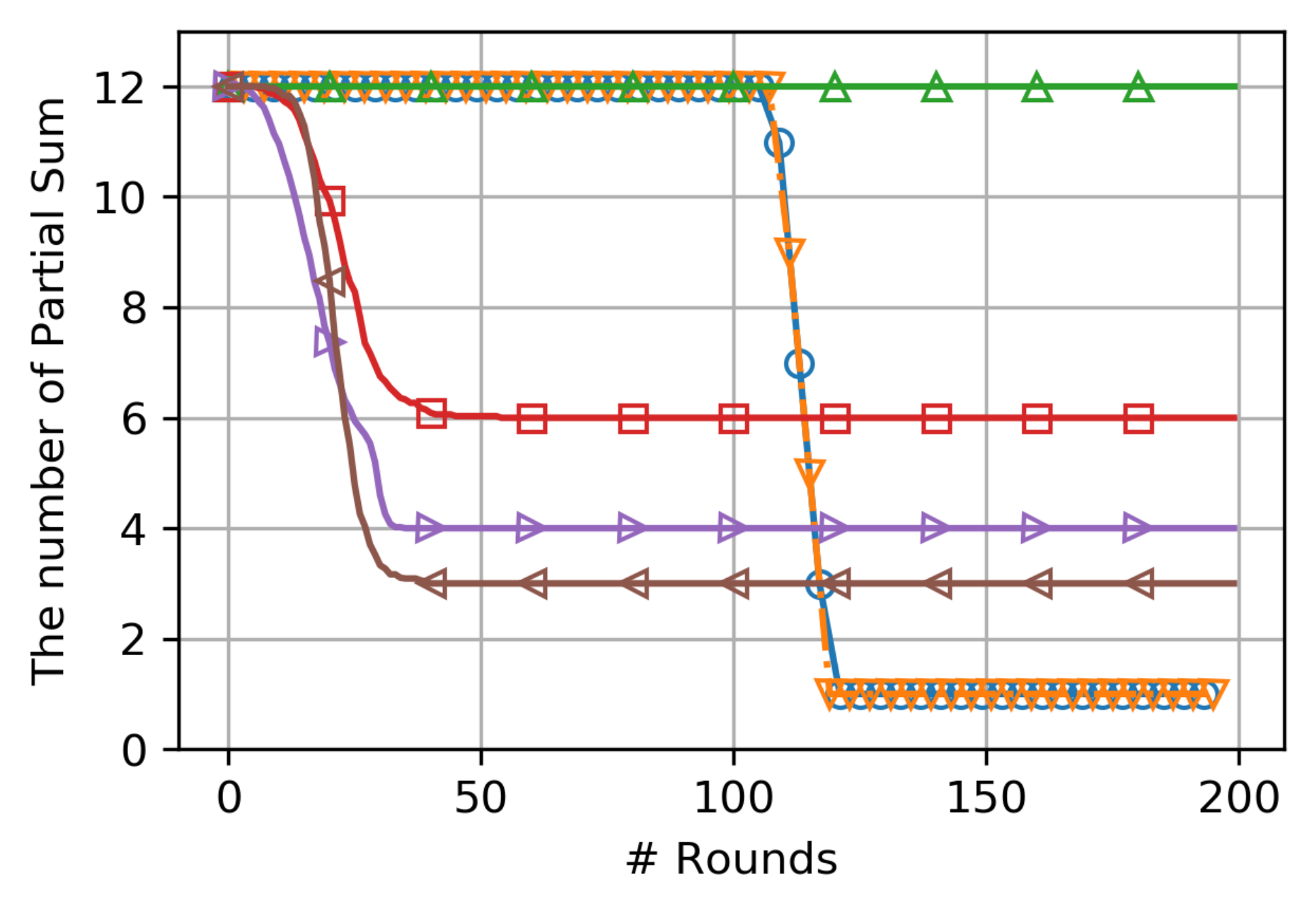}
    }
    \subfigure[Aggregation fairness gap.]{\label{fig:fairness_gap}
    \includegraphics[width=.32\textwidth]{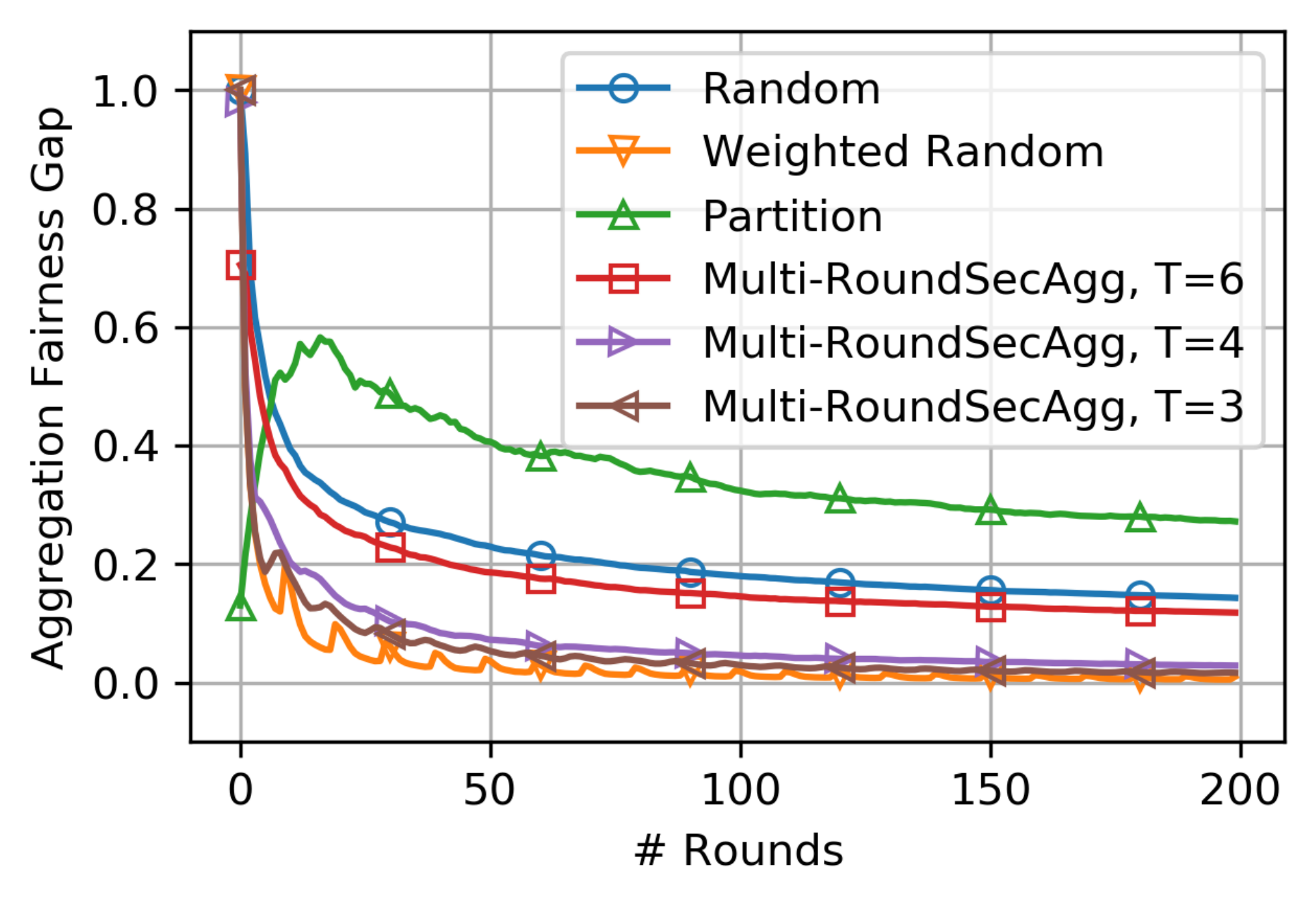}
    }
     \subfigure[Average aggregation cardinality.]{\label{fig:aggr_cardinality}
    \includegraphics[width=.32\textwidth]{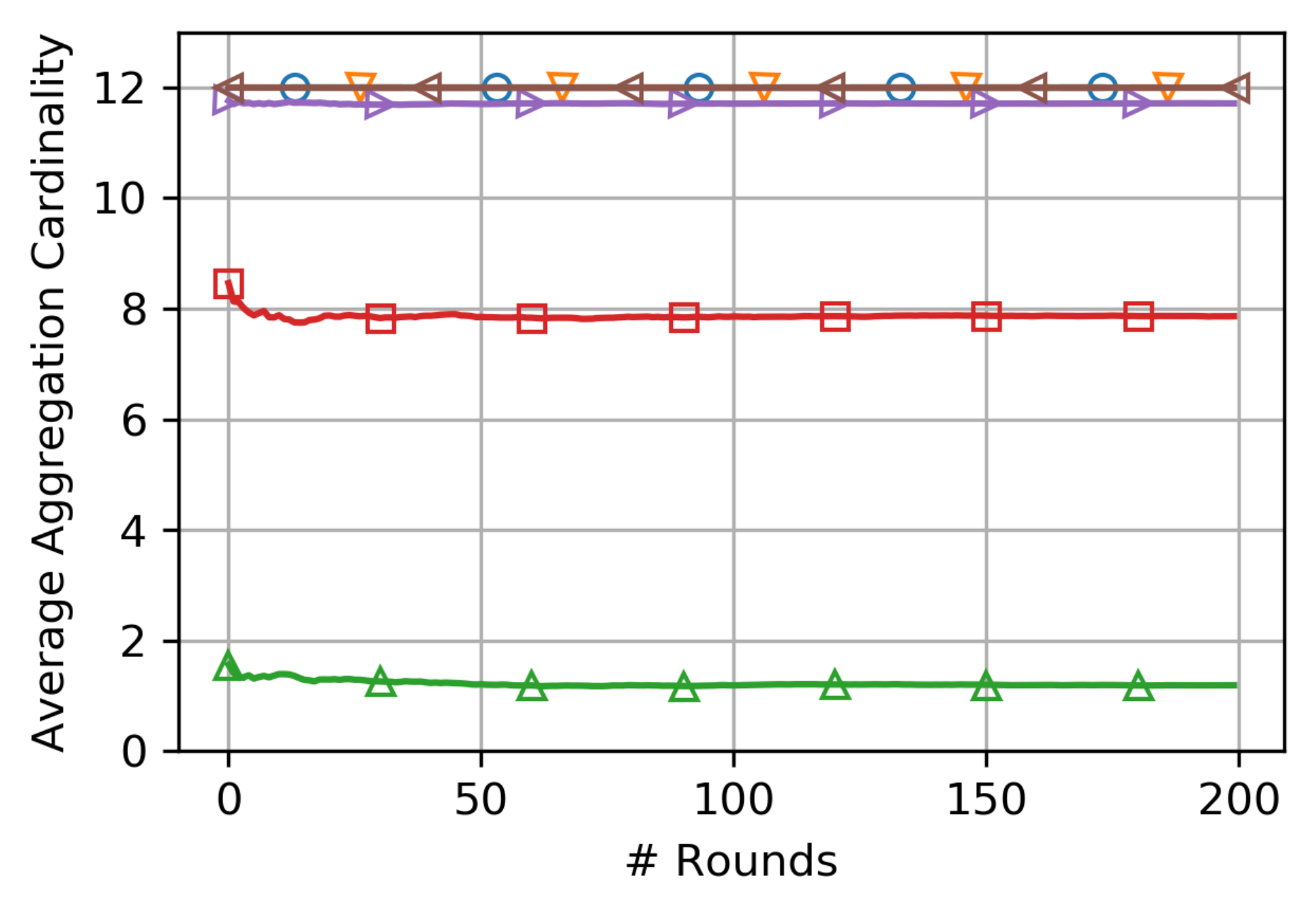}
    }
\caption{The key metrics with $N=120$ (number of users), $K=12$ (number of selected users at each round).}
\label{fig:key_metrics}
\end{figure*}

\begin{figure*}[ht!]
\centering
    \subfigure[ Privacy-aggregation cardinality trade-off.]{  \includegraphics[scale = 0.62]{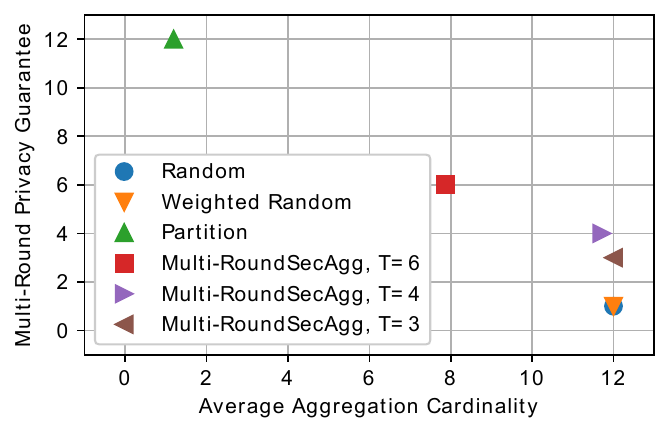} \label{fig:tradeoff}} \\ 
    \subfigure[IID data distribution.]{
        \includegraphics[scale = 0.65]{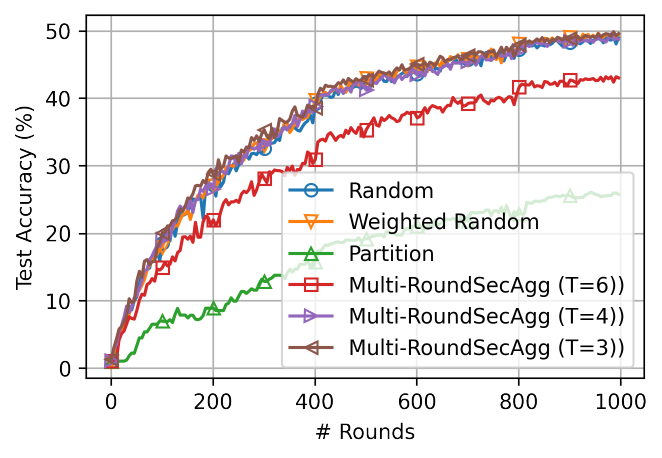} 
        \label{fig:CIFAR10_IID_accuracy}
        } 
        \subfigure[Non-IID data distribution.]{
        \includegraphics[scale =  0.65]{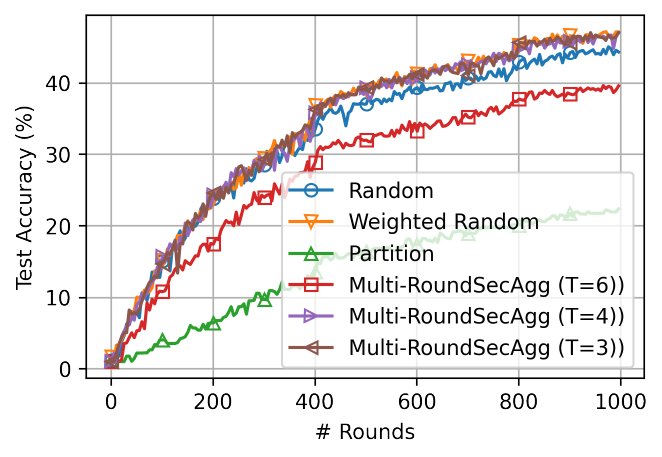} 
        \label{fig:CIFAR10_NonIID_accuracy}
        }
       \caption{Trade-off between multi-round privacy and aggregation cardinality and also the test accuracy of VGG11 in \cite{simonyan2014very} on the CIFAR-100 dataset with $N=120$ and $K=12$.}
        \label{fig:CIFAR10_accuracy}
\end{figure*}

\noindent \textbf{Setup.} We consider a FL setting with $N=120$ users, where the server aims to choose $K=12$ users at every round. 
We study two settings for partitioning the CIFAR-100 dataset.
\begin{itemize}[noitemsep,topsep=0pt,parsep=1pt,partopsep=2pt,leftmargin=*]
    \item \textbf{IID Setting.} $50000$ training samples are shuffled and partitioned uniformly across $N=120$ users. 
    \item \textbf{Non-IID Setting.} We distribute the dataset using a Dirichlet distribution~\cite{hsu2019measuring}. Specifically, for user $i \in [N]$, we sample a vector $\pmb \pi_i \sim \text{Dir}(\kappa \pmb \pi)$ with $\kappa=0.5$ and $\pmb \pi$ is the prior class distribution over the  $100$ classes. The parameter $\kappa$ controls the heterogeneity of the distributions, where $\kappa \rightarrow\infty$ results in IID setting. 
\end{itemize} 
We implement a VGG-11~\cite{simonyan2014very}, which is sufficient for our needs, as our goal is to evaluate various schemes, not to achieve the best accuracy.
%
The hyperparameters are provided in Appendix \ref{app:hyperparameters}. 

\noindent \textbf{Modeling dropouts.} To model heterogeneous system, users have different dropout probability $p_i$ selected from $\{0.1, 0.2, 0.3, 0.4, 0.5\}$.
At each round, user $i\in[N]$ drops with probability $p_i$.
%


\noindent \textbf{Implemented Schemes.} We implement the three baselines introduced in Sec. \ref{subsec:baselines}, referred to as \emph{Random}, \emph{Weighted Random}, and \emph{Partition}. For \name, we construct three privacy-preserving families with different target multi-round privacy guarantees, $T=6$, $T=4$, and $T=3$ which we refer to as \namespace($T=6$), \namespace($T=4$), and \namespace($T=3$), respectively. 
One can view the Random and Partition schemes as extreme cases of \namespace with $T=1$ and $T=K$, respectively. 
Table \ref{tbl:codebook_size} summarizes the family size $R$ defined in Section \ref{sec:Proposed}.

\noindent \textbf{Key Metrics.} To numerically demonstrate the performance of the six schemes in terms of the key metrics defined in Sec. \ref{subsec:multi-round_SA}, at each round, we measure the following metrics.
\begin{itemize}[noitemsep,topsep=0pt,parsep=1pt,partopsep=2pt,leftmargin=*]
    \item For the multi-round privacy guarantee, we measure the number of models in the partial sum that the server can reconstruct, which is given by $T^{(t)}\coloneqq\min_{\mathbfsl{z}\in \mathbb{R}^J\}} \| \mathbfsl{z}^\top \mathbf{P}^{(t)}\|_0, \text{s.t.} \  {\mathbf{P}^{(t)}}^\top\mathbfsl{z} \neq \mathbf 0$. 
    This corresponds to the weaker privacy definition of Remark~\ref{remark:weaker_T}.
    We use this weaker privacy definition as the random selection and the random weighted selection strategies provide the worst privacy guarantee even with this weaker definition, as demonstrated later. On the other hand, \namespace provides better privacy guarantees with both the strong and the weaker definitions.
    
    \item For the aggregation fairness gap, we measure the instantaneous fairness gap $F^{(t)}\coloneqq \max_{i\in[N]}F_i^{(t)} - \min_{i\in[N]}F_i^{(t)}$, where $F_i^{(t)} = \frac{1}{t+1}\sum_{l=0}^{t}\mathds{1} \big\{ \{\mathbfsl{p}^{(l)}\}_i = 1 \big\}$.
    
    \item 
    We measure the instantaneous aggregation cardinality as $C^{(t)}\coloneqq \frac{1}{t+1} \sum_{l=0}^{t} \| \mathbfsl{p}^{(l)}\|_0$.
\end{itemize}

\noindent We demonstrate these key metrics in Figure \ref{fig:key_metrics}. We make the following key observations.
\begin{itemize}[noitemsep,topsep=0pt,parsep=1pt,partopsep=2pt,leftmargin=*]
    \item \namespace achieves better multi-round privacy guarantee than both the random selection and random weighted selection strategies, while user partitioning achieves the best multi-round privacy guarantee, $T=K=12$. However, the partitioning strategy has the worst aggregation cardinality, which results in the lowest convergence rate as demonstrated later.
      

    
    \item Figure~\ref{fig:tradeoff} demonstrates the trade-off between the multi-round privacy guarantee $T$ and the average aggregation cardinality $C$. Interestingly, \namespace when $T=3$ or $T=4$ achieves better multi-round privacy guarantee than both the random selection and the weighted random selection strategies while achieving almost the same average aggregation cardinality.
    
\end{itemize}

\begin{remark} \normalfont (Multi-round Privacy of Random and Weighted Random). 
The multi-round privacy guarantees of Random and Weighted Random drop sharply as shown in Fig. \ref{fig:priavcy} as the participating matrix $\mathbf{P}^{(t)}\in\{0,1\}^{t\times N}$ becomes full rank with high probability when $t\geq N$, and hence the server can reconstruct the individual models. More precisely, Theorem \ref{thm:random-selection-multi-round}
in Appendix \ref{app:random-selection} 
shows this \emph{thresholding phenomenon}, where the probability of reconstructing individual models after certain number of rounds converges to $1$ exponentially fast. 
\end{remark} 

\noindent \textbf{Key Metrics versus Test Accuracy.}
To investigate how the  key metrics affect the test accuracy, we measure the test accuracy of the six schemes in the two settings, the IID and the non-IID settings. Our results are demonstrated in Figure \ref{fig:CIFAR10_accuracy}. 
We now make the following key observations.

\begin{itemize}[noitemsep,topsep=0pt,parsep=1pt,partopsep=2pt,leftmargin=*]
    \item In the IID setting, \namespace has a test accuracy that is comparable to the random selection and random weighted selection schemes while the \namespace schemes provide higher levels of privacy. 
    Specifically, the \namespace schemes achieve $T=3,4,6$ while the random selection and random weighted selection schemes have $T=1$.
    \item In the non-IID setting, \namespace not only outperforms the random selection scheme but also achieves a smaller aggregation fairness gap as shown in Fig.~\ref{fig:fairness_gap}.
    \item In both IID and non-IID settings, the user partitioning scheme has the worst accuracy as its average aggregation cardinality is much smaller than the other schemes as demonstrated in Fig.~\ref{fig:aggr_cardinality}. 
\end{itemize}
We provide additional experiments in App. \ref{app:exp_mnist} and  App. \ref{app:exp_ablation}.

\section{Conclusion}
\label{sec:Conclusion}
Partial user participation may breach user privacy in federated learning, even if secure aggregation is employed at every training round. 
To address this challenge, we introduced the notion of long-term privacy, which ensures that the privacy of individual models are protected over all training rounds. We developed \name, a structured user selection strategy that guarantees long-term privacy while taking into account the fairness in user selection and average number of participating users, and showed that \namespace provides a trade-off between long-term privacy and  the convergence rate. 
Our experiments on the CIFAR-100, CIFAR-$10$, and MNIST datasets on both the IID and non-IID settings show that \namespace achieves comparable accuracy to the random selection strategy (which does not ensure long-term privacy), while ensuring long-term privacy guarantees. 

 \section*{Acknowledgments}
 This material is based upon work supported by NSF grant CCF-1763673, NSF grant CCF-1909499, NSF career award CCF-2144927, ARO award W911NF1810400, ONR Award No. N00014-16-1-2189,  OUSD (R\&E)/RT\&L cooperative agreement W911NF-20-2-0267, UC Regents faculty award and gifts from Intel, Cisco, and Qualcomm. 
\bibliography{aaai23}

\onecolumn

\appendix

\noindent \textbf{Organization.} These appendices are organized as follows. 
\begin{enumerate}[label=(\Alph*)]
\item In Appendix \ref{app:privacy_proof}, we prove Theorem \ref{thm:multi-roundsecagg-guarantees}. 
\item In Appendix \ref{app:convergence_proof}, we prove Theorem \ref{thm:convergence}. \item In Appendix \ref{app:converse}, we show that batch partitioning is necessary to satisfy the multi-round privacy definition given in (\ref{eq:def_privacy}).
\item In Appendix \ref{app:algorithm}, we provide the two components of \namespace which are Algorithm \ref{alg:batch-partitioning} and Algorithm \ref{alg:proposed}. 
\item Appendix \ref{app:exp_mnist} provides additional experiments on the MNIST dataset. 
\item Appendix \ref{app:hyperparameters} provides additional details and the hyperparameters of the experiments of Section \ref{sec:Experiments} and Appendix \ref{app:exp_mnist}.
\item Appendix \ref{app:exp_ablation} provides additional experiments with various system parameters.

\item In Appendix \ref{app:random-selection}, we theoretically show that the random selection strategy discussed in Remark \ref{remark:random-selection-privacy} that aims to select $K$ available users at each round and the random selection strategy that selects the users in i.i.d fashion both have a multi-round privacy $T=1$ with high probability. We also empirically demonstrate that the local models can be reconstructed accurately when random selection is used.
\end{enumerate}

\textcolor{black}{
We list the notations in Table \ref{tbl:notations}.
}
\begin{table}[h!]
\caption{\color{black}Notations in the paper. }
\label{tbl:notations}
\begin{center}
\begin{tabular}{cc}
\toprule
Notations                    & Description \\
\midrule
$N$   & total number of users \\
$K$   & number of users selected at each iteration \\
$J$   & total number of iterations\\
$E$   & number of local iterations in each user\\
$d$   & dimension of model\\
$\mathbfsl{x}^{(t)}$  & global model at iteration $t$, $\mathbfsl{x}^{(t)}\in \mathbb{R}^d$\\
$\mathbfsl{x}^{(t)}_i$  & local model of user $i$ at iteration $t$, $\mathbfsl{x}^{(t)}_i \in \mathbb{R}^d $\\
$\mathbf{X}^{(t)}$      & concatenation of the weighted local models at iteration $t$, $\mathbf{X}^{(t)} \in \mathbb{R}^{N\times d}$\\
$\mathbfsl{p}^{(t)}$  & participation vector at iteration $t$, $\mathbfsl{p}^{(t)} \in \{0,1\}^{N}$  \\
$\mathbf{P}^{(t)}$  & participation matrix, $\mathbf{P}^{(t)} \in \{0,1\}^{t\times N}$  \\
$T$  & multi-round privacy guarantee  \\
$F$  & aggregation fairness gap  \\
$C$  & average aggregation cardinality  \\
$\mathbf{B}$  & privacy-preserving family, $\mathbf{B} \in \{0,1\}^{R_{BP}\times N}$  \\
$R_{BP}$  & the size of the privacy-preserving family of sets  \\
$\mathcal{U}^{(t)}$  & set of available users at iteration $t$  \\
$p_i$ & dropout probability of user $i$ \\
$f_i^{(t)}$ & frequency of participation of user $i$ at round $t$ \\
\bottomrule
\end{tabular}
\end{center}
\end{table}

\section{Theoretical Guarantees of \name: Proof of Theorem \ref{thm:multi-roundsecagg-guarantees}}
\label{app:privacy_proof}
In this appendix, we provide the proof of Theorem \ref{thm:multi-roundsecagg-guarantees}.
\begin{proof}
\begin{enumerate}[leftmargin=0.6cm]
\item First, we prove that \namespace ensures a multi-round privacy of $T$.  We first partition the matrix $\mathbf B$ into $R \times T$ matrices as $\mathbf B=[\mathbf B^{(1)}, \mathbf B^{(2)}, \cdots, \mathbf B^{(N/T)}]$ and the aggregated models as $\mathbf X=[ {{\mathbf X}^{(1)}}^\top, {{\mathbf X}^{(2)}}^\top, \cdots,  {{\mathbf X}^{(N/T)}}^\top]^\top$. We can then express any linear combination of the aggregated models $\mathbf X^\top \mathbf{B}^\top \mathbfsl{z}$, where $\mathbfsl z \in \mathbb R^R \setminus \{\mathbf 0\}$, as follows 
\begin{align}
 \mathbf X^\top \mathbf{B}^\top \mathbfsl{z}= \sum\limits_{i=1}^{N/T}  {\mathbf X^{(i)}}^\top {\mathbf{B}^{(i)}}^\top \mathbfsl{z}.
\end{align}
Denote the $j$-th column of $\mathbf B^{(i)}$ by $\mathbf b^{(i)}_j$ which is either a zero vector or all ones vector due to the batch partitioning structure. That is, $\mathbf b^{(i)}_j \in \{\mathbf 0, \mathbf 1\}$. Hence, $ {\mathbf{B}^{(i)}}^\top \mathbfsl{z}  \in \{\mathbf 0, a_i . \mathbf 1\}$ for some $a_i \in \mathbb R \setminus \{0\}$. Therefore, we have  
\begin{equation}
{\mathbf X^{(i)}}^\top {\mathbf{B}^{(i)}}^\top \mathbfsl{z}  = 
\begin{cases}
\quad \mathbf 0 & {\mathbf{B}^{(i)}}^\top \mathbfsl z=0,\\
a_i \sum\limits_{j=(i-1)T+1}^{iT} \mathbfsl x_j &\text{otherwise,}
\end{cases}
\end{equation}
$\forall i \in [N/T]$, which shows that \namespace achieves a multi-round privacy $T$.


\item Next, we prove that \namespace has an aggregation fairness gap $F=0$.


It is clear that the total number of times user $i$ is being selected up to time $J$ is the same as that of user $j$ who lies in the same batch as user $i$. This follows since all users in the same batch either participate together or they do not participate at all. 

It suffices to show that the \emph{expected} number of selections of user $i$ up to time $J$ is the same as that of user $j$, where user $i$ and user $j$ are in different batches. The main observation is that our protocol is \emph{symmetric}. Indeed, the only randomness in the system are the user availability randomness and the set selection randomness when there are multiple user sets satisfying the requirements. We note that for any realization of random variables such that the batch of user $i$ is selected at time $t$, there is a corresponding realization of random variables such that the batch of user $j$ is selected at time $t$ and all other selections remain exactly the same. Hence, $F_i = F_j$ for any $i\neq j$.



\item Finally, we characterize the average aggregation cardinality of \name.  The average aggregation cardinality can be expressed as follows 
\begin{align}
       \small
        &C=K \left(1-\Pr[\text{No row of} \ \mathbf B \ \text{is available}]\right) \notag \\
        &=K \left(1-\Pr[\text{At least} \ \frac{N}{T}-\frac{K}{T}+1  \ \text{batches are not available}] \right) \notag \\
        &=K\left( 1-\sum\limits_{i=N/T-K/T+1}^{N/T} \binom{N/T}{i} q^i (1-q)^{N/T-i} \right),
\end{align}
where $q$ is the probability that a certain batch is not available, which is given by $q=1-(1-p)^T$.
\end{enumerate}
\end{proof}

\section{Convergence Analysis of \namespace: Proof of Theorem \ref{thm:convergence}}
\label{app:convergence_proof}
The proof of Theorem~\ref{thm:convergence} is divided into two parts.
In the first part, we introduce a new sequence to represent the local updates in each user with respect to step index while we use the global round index $t$ for $\mathbfsl{x}^{(t)}$ in \eqref{eq:aggregation}. We carefully define the sequence and the step index, and then provide the convergence analysis of the sequence.
In the second part, we bridge the newly defined sequence and $\mathbfsl{x}^{(t)}$ in \eqref{eq:aggregation}, and provide convergence analysis of $\mathbfsl{x}^{(t)}$.

\noindent {\bf First Part (Convergence analysis of local model updates)}.

Let $\mathbfsl{w}^{(j)}_i$ be the local model updated by user $i$ at the $j$-th step. Note that this step index is different from the global round index $t$ in \eqref{eq:aggregation} as each user updates the local model by carrying out $E(\geq 1)$ local SGD steps before sending the results to the server. 
Let $\mathcal{I}_E$ be the set of global synchronization steps, i.e., $\mathcal{I}_E=\{nE|n=0,1,2,\ldots\}$.
Importantly, we define the step index $j$ such it increases from $nE$ to $nE+1$ only when the server does not skip the selection, i.e., there are at least $K$ available users at step $nE+1$ for $n\in\{0,1,2,\ldots\}$.
We denote by $\mathcal{H}_{nE}$ the set selected by \namespace at step index $nE$ and from the definition, $|\mathcal{H}_{nE}|=K$ for all $n\in\{0,1,2,\ldots\}$.
%
%
Then, the update equation can be described as 
\begin{align}
    & \mathbfsl{v}^{j+1}_i = \mathbfsl{w}^{j}_i - \eta^j \nabla L_i \left( \mathbfsl{w}^{j}_i, {\xi}^{j}_i \right),\\
    & \mathbfsl{w}^{j+1}_i =
    \left\{
    \begin{array}{ll}
        \mathbfsl{v}^{j+1}_i & \text{if \quad } j+1 \in \mathcal{I}_E\\
        \frac{1}{K} \sum_{k\in\mathcal{H}_{j+1}} \mathbfsl{v}^{j+1}_k  & \text{if \quad } j+1 \notin \mathcal{I}_E
    \end{array},
    \right.
\end{align}
where we introduce an additional variable $\mathbfsl{v}^{j+1}_i$ to represent the immediate result of one step SGD from $\mathbfsl{w}^{j}_i$. We can view $\mathbfsl{w}^{j+1}_i$ as the model obtained after aggregation step (when $j+1$ is a global synchronization step).
%
Motivated by \cite{stich2018local,li2019convergence}, we define two virtual sequences
\begin{align}
    & \overline{\mathbfsl{v}}^{j} = \frac{1}{N}\sum_{i=1}^{N}{\mathbfsl{v}}^{j}_i, \\
    & \overline{\mathbfsl{w}}^{j} = \frac{1}{N}\sum_{i=1}^{N}{\mathbfsl{w}}^{j}_i. \label{eq:def_w_bar}
\end{align}
We can interpret $\overline{\mathbfsl{v}}^{j+1}$ as the result of single step SGD from $\overline{\mathbfsl{w}}^{j}$. When $j \notin \mathcal{I}_E$, both $\overline{\mathbfsl{v}}^{j}$ and $\overline{\mathbfsl{w}}^{j}$ are not accessible. 
We also define $\overline{\mathbfsl{g}}^{j}=\frac{1}{N}\sum_{i=1}^{N} \nabla L_i\left( \mathbfsl{w}^j_i\right)$ and $\mathbfsl{g}^j = \frac{1}{N}\sum_{i=1}^{N} \nabla L_i\left( \mathbfsl{w}^j_i, \xi^j_i\right)$. Then, $\overline{\mathbfsl{v}}^{j+1} = \overline{\mathbfsl{w}}^{j}-\eta^j \mathbfsl{g}^j$.

Now, we state our two key lemmas.
\begin{lemma}[Unbiased selection]\label{lem:unbiased_selection} When $j+1\in\mathcal{I}_E$, the following is satisfied,
    \begin{equation}
        \mathbb E_{\mathcal{H}_{j+1}} [ \overline{\mathbfsl{w}}^{j+1} ] = \overline{\mathbfsl{v}}^{j+1}.
    \end{equation}
\end{lemma}
\begin{proof}
    Let $\mathcal{H}_{j+1}=\{i_1,\ldots,i_K\}$. Then, we have
    \begin{align}
        \mathbb E_{\mathcal{H}_{j+1}} [ \overline{\mathbfsl{w}}^{j+1} ]
        = \frac{1}{K} E_{\mathcal{H}_{j+1}} \left[ \sum_{k\in\mathcal{H}_{j+1}} \mathbfsl{v}^{j+1}_k \right] \notag 
        = \frac{1}{K} E_{\mathcal{H}_{j+1}} \left[ \sum_{k=1}^K \mathbfsl{v}^{j+1}_{i_k}\right] \notag 
        &= E_{\mathcal{H}_{j+1}} [\mathbfsl{v}^{j+1}_{i_k}] \notag\\
        &= \sum_{k=1}^{N} \frac{1}{N} \mathbfsl{v}^{j+1}_{k} = \overline{\mathbfsl{v}}^{j+1} \label{eq:lemma1_proof}
    \end{align}
    where \eqref{eq:lemma1_proof} follows as $\Pr[i_k=j]=\frac{1}{N}$ for $i\in[N]$. This is because the sampling probability of each user is identical due to the symmetry in the construction and the fact that all users have the same dropout probability.
\end{proof}



Now, we provide the convergence analysis of the sequence $\overline{\mathbfsl{w}}^{j}$ defined in \eqref{eq:def_w_bar}. We have,
\begin{align}
    \lVert \overline{\mathbfsl{w}}^{j+1} - \mathbfsl{w}^{*} \rVert^2
    &= \lVert \overline{\mathbfsl{w}}^{j+1} - \overline{\mathbfsl{v}}^{j+1} + \overline{\mathbfsl{v}}^{j+1} - \mathbfsl{w}^{*} \rVert^2 \notag \\
    &= \lVert \overline{\mathbfsl{w}}^{j+1} - \overline{\mathbfsl{v}}^{j+1} \rVert^2
        + \lVert \overline{\mathbfsl{v}}^{j+1} - \mathbfsl{w}^{*} \rVert^2
        + 2 \left( \overline{\mathbfsl{w}}^{j+1} - \overline{\mathbfsl{v}}^{j+1} \right)^{\top} \left( \overline{\mathbfsl{v}}^{j+1} - \mathbfsl{w}^{*} \right). \label{eq:key_equality}
\end{align}
When the expectation is taken over $\mathcal{H}_{j+1}$, the last term in \eqref{eq:key_equality} becomes zero due to Lemma \ref{lem:unbiased_selection}.
For the second term in \eqref{eq:key_equality}, we have 
\begin{equation}\label{eq:second_term}
    \lVert \overline{\mathbfsl{v}}^{j+1} - \mathbfsl{w}^{*} \rVert^2 
    \leq (1-\eta^{j} \mu) \lVert \overline{\mathbfsl{w}}^{j} - \mathbfsl{w}^{*} \rVert^2 
    + \alpha ({\eta^{j}})^2,
\end{equation}
where $\alpha=\frac{1}{N}\sum_{i=1}^{N}\sigma_i^2 + 6\rho\Gamma + 8(E-1)^2G^2$ and \eqref{eq:second_term} directly follows from Lemma $1,2,3$ of \cite{li2019convergence}.
The first term in \eqref{eq:key_equality} becomes zero if $j+1\in\mathcal{I}_E$, and if $j+1\notin\mathcal{I}_E$, from Lemma $5$ of \cite{li2019convergence}, it is bounded by 
\begin{equation}\label{eq:third_term}
    \mathbb{E}_{\mathcal{H}_{j+1}} \lVert \overline{\mathbfsl{w}}^{j+1} - \overline{\mathbfsl{v}}^{j+1} \rVert^2 \leq \beta ({\eta^{j}})^2,
\end{equation}
where $\beta=\frac{4(N-K)E^2G^2}{K(N-1)}$. By combining \eqref{eq:key_equality} to \eqref{eq:third_term}, we have
\begin{equation}
    \mathbb{E} \lVert \overline{\mathbfsl{w}}^{j+1} - \mathbfsl{w}^{*} \rVert^2
    \leq (1-\eta^{j} \mu) \lVert \overline{\mathbfsl{w}}^{j} - \mathbfsl{w}^{*} \rVert^2 
    + (\alpha+\beta)({\eta^{j}})^2.
\end{equation}
%
Then by utilizing the similar induction in \cite{li2019convergence}, we can show that
\begin{equation}\label{eq:induction}
    \mathbb{E} \lVert \overline{\mathbfsl{w}}^{j+1} - \mathbfsl{w}^{*} \rVert^2 \leq
    \frac{1}{\gamma + t -1}\left( \frac{4(\alpha+\beta)}{\mu^2} + \gamma \mathbb{E} \lVert \overline{\mathbfsl{w}}^{0} - \mathbfsl{w}^{*} \rVert^2 \right),
\end{equation}
where $\gamma=\max\left\{\frac{8\rho}{\mu},E\right\}$. By combining \eqref{eq:induction} with $\rho$-smoothness of the global loss function in \eqref{eq:objective_fnc}, we have
\begin{equation}\label{eq:conv_w_bar}
    \mathbb{E}[L(\overline{\mathbfsl{w}}^{I})]-L^{*} \leq \frac{\rho}{\gamma + I -1}
        \left( \frac{2(\alpha + \beta)}{\mu^2} + \frac{\gamma}{2} \mathbb{E} \lVert \overline{\mathbfsl{w}}^{0} - \mathbfsl{x}^{*}\rVert ^2 \right).
\end{equation}

\noindent {\bf Second Part (Convergence analysis of global model)}.

Now, we bridge the sequence $\overline{\mathbfsl{w}}^{T}$ and $\mathbfsl{x}^{(t)}$ in \eqref{eq:aggregation} to provide the convergence analysis of $\mathbfsl{x}^{(t)}$.
Since we define the step index $j$ such that $j$ increases from $nE$ to $nE+1$ only when the server does not skip the selection, we have
\begin{equation}\label{eq:bridge}
    \mathbb{E} [L(\mathbfsl{x}^{(J)})] = \mathbb{E} [L(\overline{\mathbfsl{w}}^{(JE \upphi)})]
\end{equation}
where $\upphi$ is the probability that there are at least $K$ available users at a certain synchronization step, and $\upphi=\frac{C}{K}$ due to the fact that $C=K\cdot \Pr[\text{at least one row of} \ \mathbf B \ \text{is available}]=K\upphi$.
By combining~\eqref{eq:conv_w_bar} and~\eqref{eq:bridge}, we have that, 
\begin{equation}
    \mathbb{E}[L(\mathbfsl{x}^{(J)})]-L^{*} \leq \frac{\rho}{\gamma + \frac{C}{K}EJ-1}
    \left( \frac{2(\alpha + \beta)}{\mu^2} + \frac{\gamma}{2} \mathbb{E} \lVert \mathbfsl{x}^{(0)} - \mathbfsl{x}^{*}\rVert ^2 \right),
\end{equation}
which completes the proof.

\section{Necessity of Batch Partitioning (BP)}
In this appendix, we show that batch partitioning is necessary to satisfy the multi-round privacy guarantee of Equation (\ref{eq:def_privacy}) and our strategy is optimal in the sense that no other strategy can have more distinct user selection sets than our strategy. 
\begin{proof}
Consider any scheme which selects sets from an $R \times N$ matrix  $\mathbf V=[\mathbfsl v_1, \cdots, \mathbfsl v_N]^\top$. Denote the linear coefficients multiplying them by $z_i$, $i \in [R]$. Then, the $i$-th element of $\mathbf V^\top \mathbfsl z$ is given by
\begin{align}
\label{eqn:star}
\{\mathbf V^\top \mathbfsl z\}_i=\sum_{j \in \mathrm{supp}(\mathbfsl v_i)} z_i. 
\end{align}
We now claim that we can cluster the entries using equivalence of linear functions to groups, where each group must have a size of at least $T$ except for the group corresponding to the zero function. To show this, we choose each $z_i \stackrel{\text{i.i.d.}}{\sim} U[0,1]$, and the key observation is that if two entries have different linear functions then their final value after this assignment would be different with probability one. Since the scheme satisfies a multi-round privacy $T$, this  implies that for each non-zero linear function of the form of Equation (\ref{eqn:star}), there must be at least $T$ of them. If we group the entries according to the equivalence of linear functions, we get at most $N/T$ groups (ignoring the group of constant zero).

Then, we show that the total number of possible sets $R$ is upper-bounded by ${N/T \choose K/T}$. We observe that the total number of non-zero groups we can choose for each vector is at most $K/T$ due to the size of each group, so the total number of distinct vectors satisfying the weight requirement is at most
\begin{align}
\label{eqn:double-star}
   R \leq R_{\textrm{max}} \deff \binom{D}{E}, 
\end{align}
where $D \leq N/T$ is the total number of groups corresponding to the non-zero linear functions, and $E\leq K/T$ is the total number of groups we may select in each round. Next, we have
\begin{align}
R_{\textrm{max}} &= \binom{D}{E} \notag \\&\overset{(i)}{\leq} \binom{N/T}{E} \notag \\ &\overset{(ii)}{\leq}  \binom{N/T}{K/T}=R_{\textrm{BP}},
\end{align}
where $(i)$ follows since $\binom{D}{E}$ is monotonically increasing w.r.t $D$, and $(ii)$ follows as $\binom{D}{E}$ is monotonically increasing w.r.t $E$ if $E \leq D/2$. 
\end{proof}
\label{app:converse}

\section{The Two Components of \namespace: Algorithms \ref{alg:batch-partitioning} and \ref{alg:proposed}}~\label{app:algorithm}
We describe the two components of \namespace in detail in Algorithm \ref{alg:batch-partitioning} and Algorithm \ref{alg:proposed}.
\begin{algorithm}
	\caption{Batch Partitioning Privacy-preserving Family Generation} \label{alg:batch-partitioning}
	 \textbf{Input:} Number of users $N$, row weight $K$ and the desired multi-round privacy guarantee $T$. \\
 \textbf{Output:} Privacy-preserving Family $\mathbf B \in \{0,1\}^{R_{\textrm{BP}} \times N}$, where $R_{\textrm{BP}}=\binom{N/T}{K/T}$ \\
 \textbf{Initialization:} $\mathbf{B}=\mathbf{0}_{R_{\textrm{BP}} \times N}$.
	\begin{algorithmic}[1]
	    \State Partition index sets $\{1,2,\ldots,N\}$ into $\frac{N}{T}$ sets, $\mathcal{G}_1,\ldots,\mathcal{G}_{\frac{N}{T}}$, where $|\mathcal{G}_i|=T$ for all $i\in[\frac{N}{T}]$.
	    \State Generate all possible sets each of which is union of $\frac{K}{T}$ sets out of $\frac{N}{T}$ sets ($\mathcal{G}_1,\ldots,\mathcal{G}_{\frac{N}{T}}$) without replacement. Denote the generated sets by $\mathcal{L}_1,\ldots,\mathcal{L}_{R_{\textrm{BP}}}$.
		\For {$i=1,2,\ldots, R_{\textrm{BP}}$}
		    \For{$j=1,2,\ldots,N$}
		        \If{$j\in\mathcal{L}_i$}  
		             $\{\mathbfsl{b}_i\}_j=1$ 
		        \EndIf
		    \EndFor
		\EndFor
	\end{algorithmic} 
\end{algorithm}

\begin{algorithm}
	\caption{Available Batch Selection} \label{alg:proposed}
	 \textbf{Input:} A family of sets $\mathbf B$, set of available users $\mathcal{U}^{(t)}$, the frequency of participation vector $\mathbfsl f^{(t-1)}$, and the selection mode $\lambda$. \Comment{$\lambda=0$ when $p_i=p, \forall i \in [N]$ and $1$ otherwise} \\
 \textbf{Output:} A participation vector $\mathbfsl p^{(t)}$. \\
 \textbf{Initialization:} $\mathbf{B}^{(t)}=[ \ ], \ell^{(t-1)}_{\textrm{min}}\coloneqq \argmin_{i \in \mathcal U^{(t)}} f^{(t-1)}_i$.e
	\begin{algorithmic}[1]
		\For {$i=1,2,\ldots, R_{\textrm{BP}}$}
          	\If {$\mathrm{supp(\mathbfsl b_i)} \subseteq \mathcal{U}^{(t)}$}
		       $\mathbf {B}^{(t)}=[\mathbf {B^{(t)}}^\top, \mathbfsl b_i]^\top$.
		    \EndIf			
		\EndFor
         \If {$\mathbf {B}^{(t)}=[ \ ]$}
		      \State $\mathbfsl b_{r(t)}^{(t)} =\mathbf 0$.
	     \ElsIf{$\lambda=0$} \Comment{Uniform selection}
	        \State Select a row from $\mathbf{B}^{(t)}$, $\mathbfsl b_{r(t)}^{(t)}$, uniformly at random. 
		 \Else \Comment{Fairness-aware selection}
		  \State Select a row from $\mathbf{B}^{(t)}$, $\mathbfsl b_{r(t)}^{(t)}$, uniformly at random from the rows that include $\ell^{(t-1)}_{\textrm{min}}$. 
		 \EndIf
	 \State $\mathbfsl p^{(t)}=\mathbfsl b_{r(t)}^{(t)}$.
	 \State Update $\mathbfsl f^{(t)}=\mathbfsl f^{(t-1)}+\mathbfsl p^{(t)}$
	\end{algorithmic} 
\end{algorithm}

\color{black}

\section{Additional Experiments: MNIST dataset} \label{app:exp_mnist}
\noindent\textbf{MNIST.}
To further investigate the performance of \name, we implement a simple CNN \cite{mcmahan2016communication} with two $5\times5$ convolution layers, a fully connected layer with ReLU activation, and a final Softmax output layer. This standard model has $1,\!663,\!370$ parameters and is sufficient for our needs, as our goal is to evaluate various schemes, not to achieve the best accuracy. 
We study the two settings for partitioning the MNIST dataset across the users.
\begin{itemize}[leftmargin=0.6cm]
    \item \textbf{IID Setting.} In this setting, the $60000$ training samples are shuffled and partitioned uniformly across the $N=120$ users, where each user receives $500$ samples. 
    %
    %
    \item \textbf{Non-IID Setting.} In this setting, we first sort the dataset by the digit labels, partition the sorted dataset into $120$ shards of size $500$, and assign each of the  $120$ users one shard. 
    This is similar to the pathological non-IID partitioning setup proposed in \cite{mcmahan2016communication}, where our partition is an extreme case as each user has only one digit label while each user in \cite{mcmahan2016communication} has two. 
    %
    %
\end{itemize}
    

\begin{figure*}[th!]%
    \centering
    \subfigure[IID data distribution.]{%
    \vspace{-0.2cm}
    \label{fig:MNIST_IID_accuracy}%
    \includegraphics[height=1.65in, trim=0.1cm 0.1cm 0.15cm 0.15cm, clip]{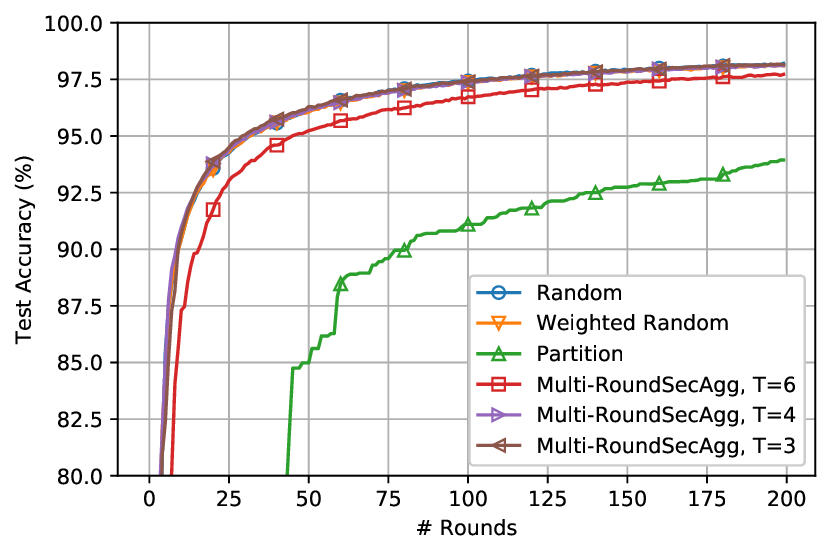}}%
    \;\;
    \subfigure[Non-IID data distribution.]{%
    \vspace{-0.2cm}
    \label{fig:MNIST_NonIID_accuracy}%
    \includegraphics[height=1.65in, trim=0.1cm 0.1cm 0.15cm 0.15cm, clip]{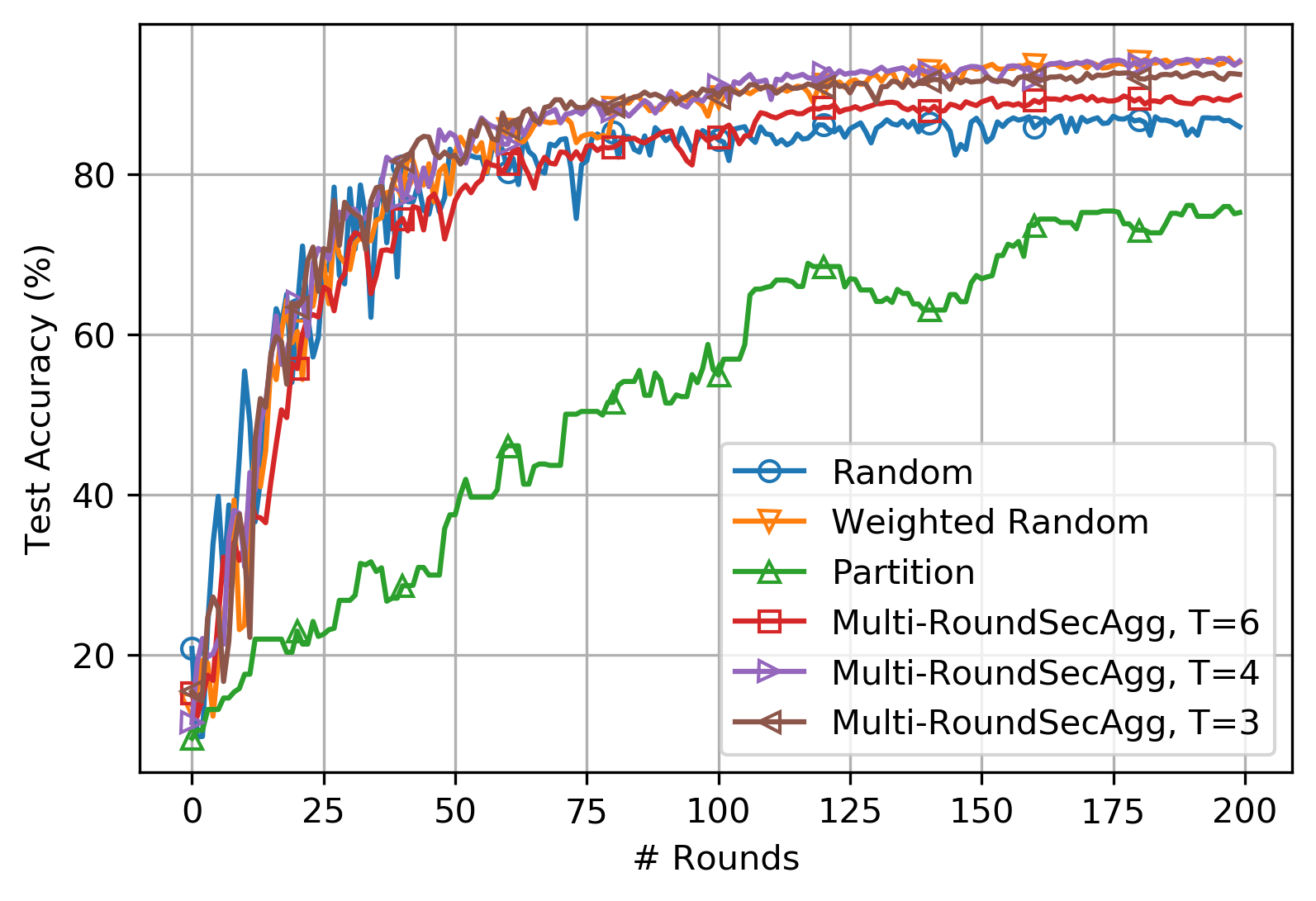}}%
    \vspace{-0.2cm}
    \caption{Training rounds versus test accuracy of CNN in \cite{mcmahan2016communication} on the MNIST with $N=120$ and $K=12$.}
\label{fig:MNIST_accuracy}
\vspace{-0.2cm}
\end{figure*}

\noindent\textbf{CIFAR-10} We also consider both IID and Non-IID distribution, and implement LeNet~\cite{lecun1999object} for both setting. While the state-of-the-art models \cite{kolesnikov2019big,tan2019efficientnet} achieve $99\%$ accuracy, LeNet is sufficient for our needs, as our goal is to evaluate various schemes, not to achieve the best accuracy.
\begin{itemize}[leftmargin=*]
    \item \textbf{IID Setting.} In this setting, the $50000$ training samples are shuffled and partitioned uniformly across the $N=120$ users, where each user receives $417$ or $416$ samples.
    \item \textbf{Non-IID dataset.} In this setting, we utilize the \emph{data-sharing strategy} of \cite{zhao2018federated}, where the $50000$ training samples are divided into a globally shared dataset $\mathcal{G}$ and private dataset $\mathcal{D}$.
    We set $|\mathcal{G}|=200$ and $|\mathcal{D}|=49800$. 
    Then, we sort $\mathcal{D}$ by the labels, partition it into $120$ shards of size $415$, and assign each of the $120$ users one shard.
    Each user has $200$ samples of globally shared data and $415$ samples of private dataset with one label. 
\end{itemize}
%
%


\begin{figure*}[th!]%
    \centering
    \subfigure[IID data distribution.]{%
    \vspace{-0.2cm}
    \label{fig:cifar100_IID_accuracy}%
    \includegraphics[height=1.65in, trim=0.1cm 0.1cm 0.15cm 0.15cm, clip]{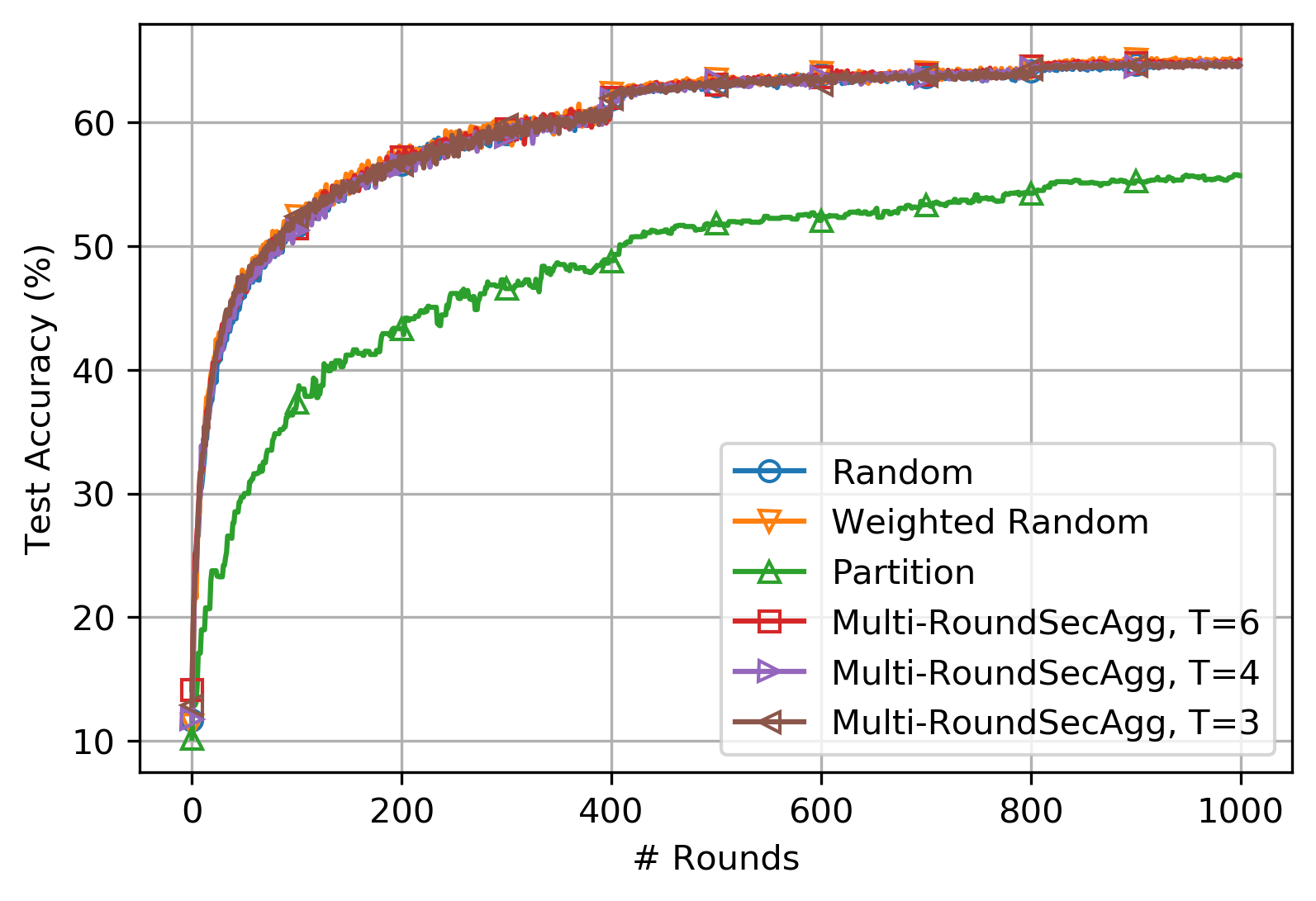}}%
    \;\;
    \subfigure[Non-IID data distribution.]{%
    \vspace{-0.2cm}
    \label{fig:cifar100_NonIID_beta0.5_accuracy}%
    \includegraphics[height=1.65in, trim=0.1cm 0.1cm 0.15cm 0.15cm, clip]{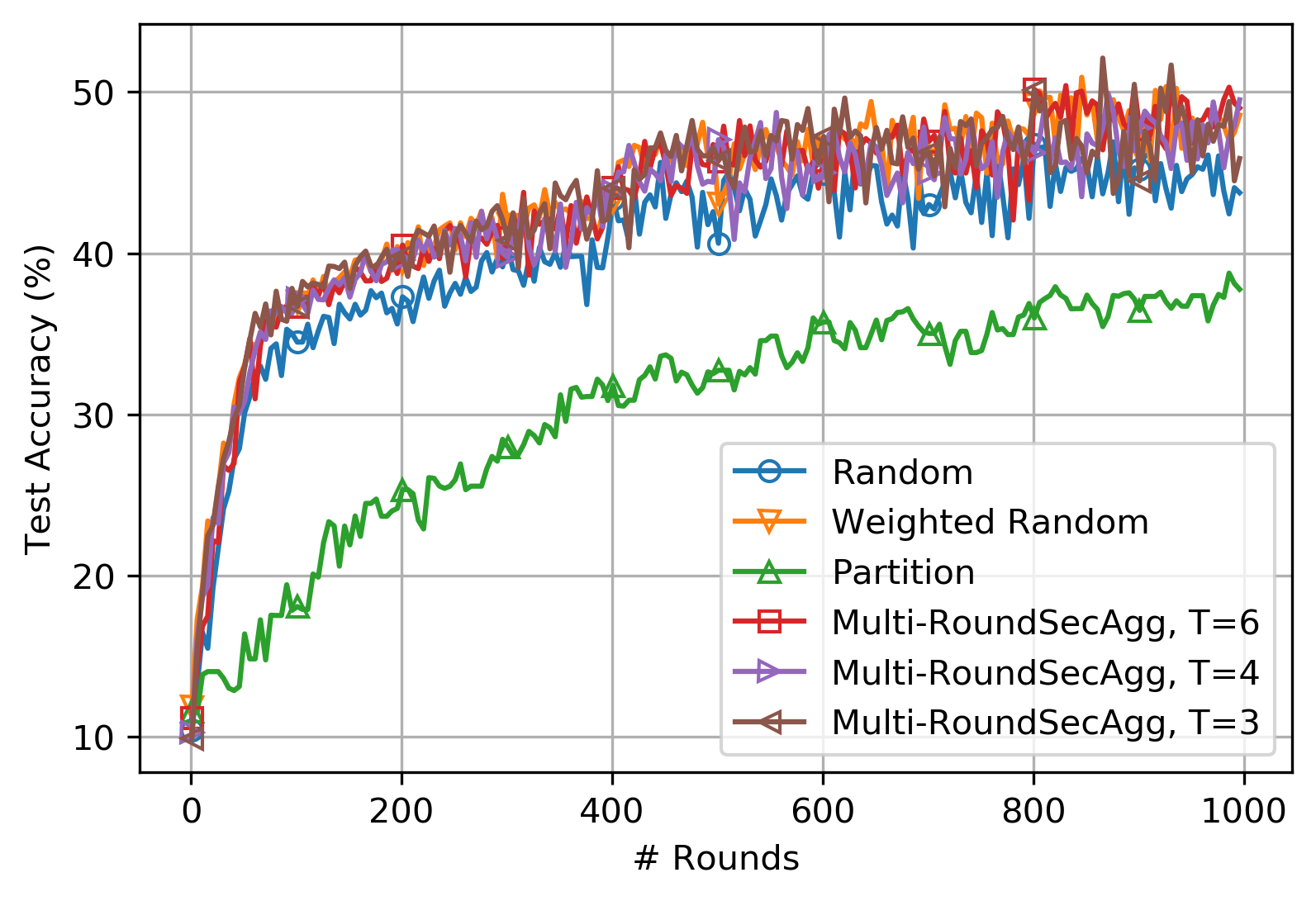}}%
    \vspace{-0.2cm}
    \caption{Training rounds versus test accuracy of LeNet in \cite{lecun1999object} on the CIFAR-10 with $N=120$ and $K=12$.}
\label{fig:cifar100_accuracy}
\vspace{-0.2cm}
\end{figure*}

We measure the test accuracy of the six schemes on the MNIST and CIFAR-10 dataset with the two distribution settings, the IID and the Non-IID.
Our results are demonstrated in Figure~\ref{fig:MNIST_accuracy} and Figure~\ref{fig:cifar100_accuracy}.
We make the following key observations, which are similar to the observations on the CIFAR-100 dataset.

\begin{itemize}[leftmargin=0.6cm]
    \item In the IID setting, the \namespace schemes show comparable test accuracy to the random selection and random weighted selection schemes while the \namespace schemes provide better multi-round privacy guarantee $T$. 
    
    \item In the non-IID setting, the \namespace schemes outperform the random selection scheme while showing comparable test accuracy to the weighted random selection scheme.
    This is because \namespace schemes have better aggregation fairness gaps as demonstrated in Figure \ref{fig:fairness_gap}, which results in  better test accuracy in the non-IID setting.
    
    \item In both IID and non-IID settings, the user partitioning scheme has the worst test accuracy as its average aggregation cardinality is much smaller than the other schemes. 
\end{itemize}

\section{Experiment Details} \label{app:hyperparameters}
In this section, we provide more details about the experiments of Section \ref{sec:Experiments} and Appendix \ref{app:exp_mnist}.

We summarize the test accuracy of CIFAR-$100$, CIFAR-$10$, and MNIST dataset in Table~\ref{tbl:accurac_CIFAR100}, Table~\ref{tbl:accurac_CIFAR10} and Table~\ref{tbl:accuracy_MNIST}, respectively. For all datasets, we run experiments five times with different random seeds and present the average value of the test accuracy in Table~\ref{tbl:accurac_CIFAR10} and Table~\ref{tbl:accuracy_MNIST}.

\begin{table}[t!]
\small
\caption{Test accuracy of VGG11 in \cite{simonyan2014very} on the CIFAR-$100$ dataset with $N=120$ and $K=12$. }
\label{tbl:accurac_CIFAR100}
\begin{center}
\begin{tabular}{cccc}
\toprule
Scheme                    & IID Setting  & Non-IID Setting\\
\midrule
Random selection          &  $49.15\%$ &  $44.32\%$ \\
Weighted random selection &  $50.06\%$ &  $47.11\%$ \\
User partition            &  $25.73\%$ &  $22.32\%$ \\
\name, T=6                &  $42.89\%$ &  $39.57\%$ \\
\name, T=4                &  $49.43\%$ &  $46.99\%$ \\
\name, T=3                &  $50.22\%$ &  $47.06\%$ \\
\bottomrule
\end{tabular}
\end{center}
\end{table}

\begin{table}[t!]
\small
\caption{Test accuracy of LeNet in \cite{lecun1999object} on the CIFAR-$10$ dataset with $N=120$ and $K=12$. }
\label{tbl:accurac_CIFAR10}
\begin{center}
\begin{tabular}{cccc}
\toprule
Scheme                    & IID Setting  & Non-IID Setting\\
\midrule
Random selection          &  $64.64\%$ &  $45.20\%$ \\
Weighted random selection &  $65.06\%$ &  $47.89\%$ \\
User partition            &  $55.70\%$ &  $37.74\%$ \\
\name, T=6                &  $65.01\%$ &  $46.35\%$ \\
\name, T=4                &  $64.95\%$ &  $47.00\%$ \\
\name, T=3                &  $64.80\%$ &  $47.21\%$ \\
\bottomrule
\end{tabular}
\end{center}
\end{table}

\begin{table}[t!]
\small
\caption{Test accuracy of the CNN in \cite{mcmahan2016communication} on the MNIST dataset with $N=120$ and $K=12$. }
\label{tbl:accuracy_MNIST}
\begin{center}
\begin{tabular}{cccc}
\toprule
Scheme                    & IID Setting  & Non-IID Setting\\
\midrule
Random selection          &  $98.21\%$ &  $85.79\%$ \\
Weighted random selection &  $98.10\%$ &  $94.04\%$ \\
User partition            &  $93.94\%$ &  $75.26\%$ \\
\name, T=6                &  $97.72\%$ &  $89.88\%$ \\
\name, T=4                &  $98.11\%$ &  $92.51\%$ \\
\name, T=3                &  $98.15\%$ &  $94.16\%$ \\
\bottomrule
\end{tabular}
\end{center}
\end{table}

\noindent {\bf Hyperparameters and computing resources.} For a fair comparison between $6$ schemes, we find the best learning rate from $\{0.1, 0.03, 0.01, 0.003, 0.001, 0.0003, 0.0001\}$. 
Given the choice of the best learning rate $\eta$, $\eta$ is decayed to $0.4\eta$ every $400$ and $800$ rounds to train the LeNet on the CIFAR-$10$ dataset or train VGG11 on the CIFAR-100 dataset while $\eta$ is not decayed in the CNN on the MNIST dataset.
To train the LeNet on the CIFAR-$10$ dataset or train VGG11 on the CIFAR-100 dataset, we use the mini-batch size of $50$ and $E=1$ local epoch for both IID and Non-IID settings.
To train the CNN on the MNIST dataset, we use the mini-batch size of $100$ and $E=1$ local epoch for both IID and Non-IID settings.
All experiments are conducted with users equipped with $3.4$ GHz $4$ cores i-$7$ Intel CPU and NVIDIA Geforce $1080$, and the users communicate amongst each other through Ethernet to transfer the model parameters.

\section{Additional Experiments: Ablation Study} \label{app:exp_ablation}
In this Appendix, we further investigate the performance of \name with various settings of the system design parameters, the number of total users($N$), the number of selected users per round ($K$), and target multi-round privacy guarantee($T$). We use the same dropout model as Section \ref{sec:Experiments}, i.e., considering heterogeneous environments where users have different dropout probability among $\{0.1, 0.2, 0.3, 0.4, 0.5\}$.
We implement LeNet \cite{lecun1999object} for image classification for CIFAR-10 with IID distribution.

\begin{figure*}[th!]%
    \centering
    \subfigure[$(N,K)=(240,12)$]{%
    \vspace{-0.2cm}
    \label{fig:CIFAR10_IID_N240_K12}%
    \includegraphics[height=1.65in, trim=0.1cm 0.1cm 0.15cm 0.15cm, clip]{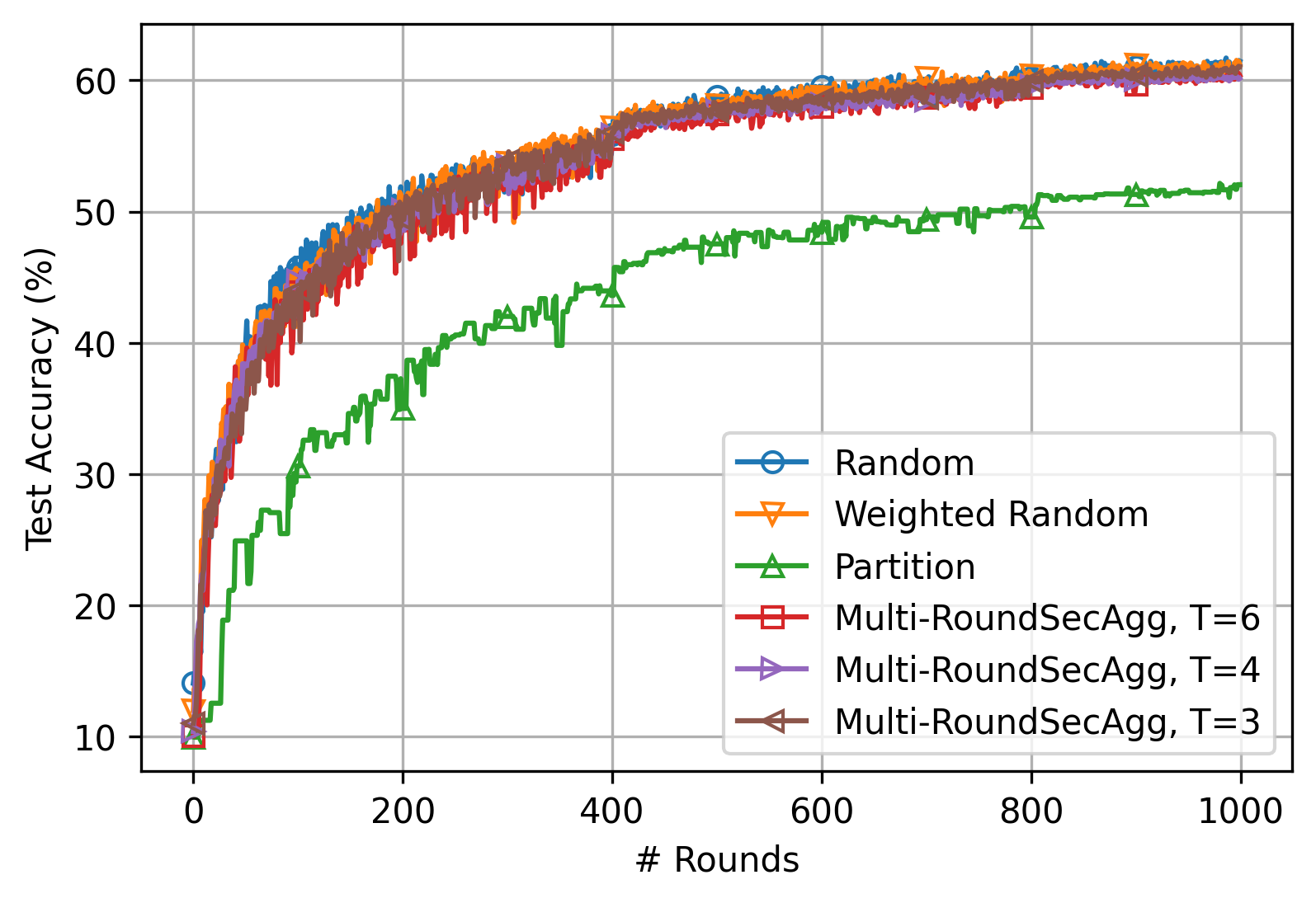}}%
    \;\;
    \subfigure[$(N,K)=(120,24)$]{%
    \vspace{-0.2cm}
    \label{fig:CIFAR10_IID_N120_K24}%
    \includegraphics[height=1.65in, trim=0.1cm 0.1cm 0.15cm 0.15cm, clip]{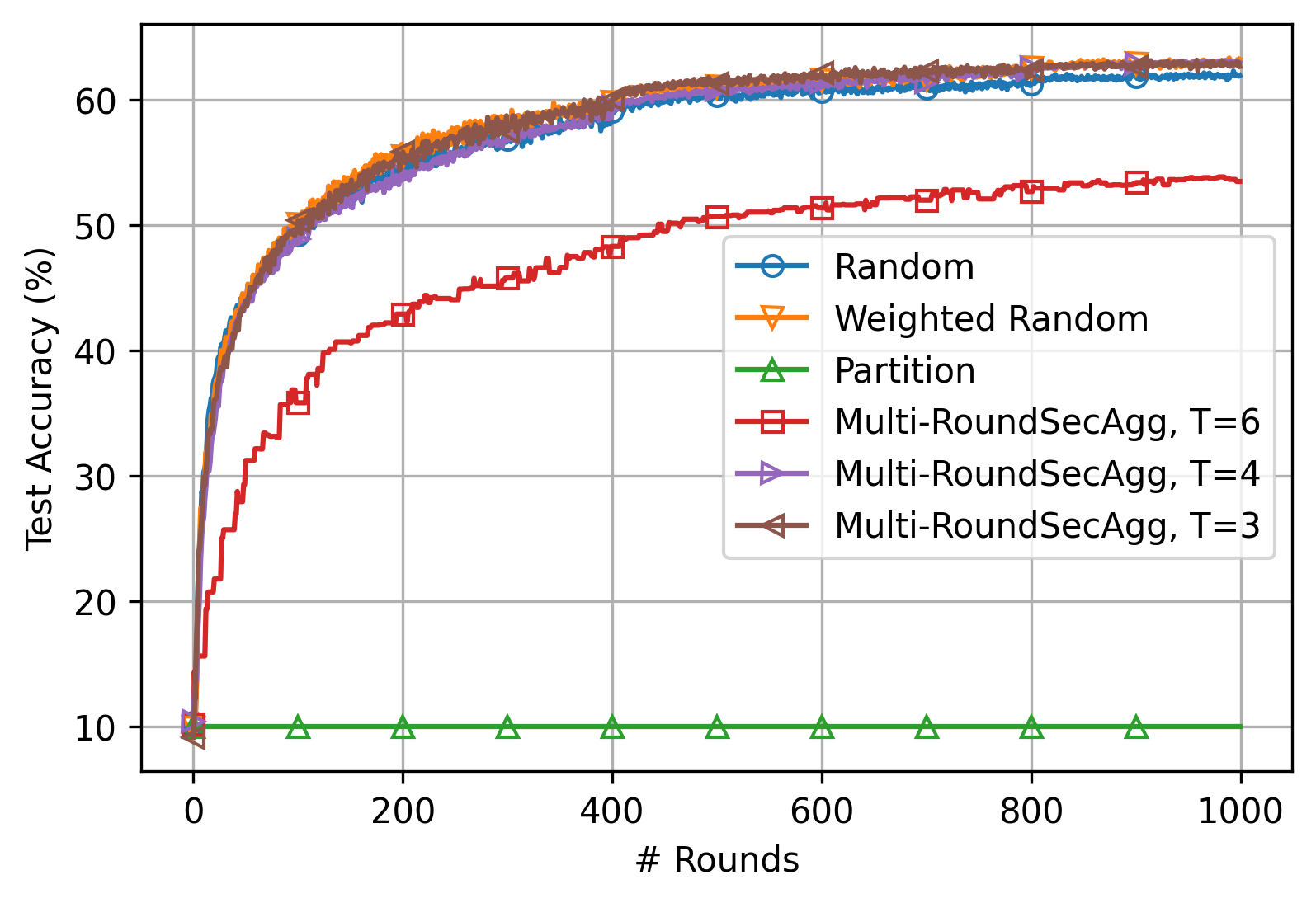}}%
    \vspace{-0.2cm}
    \caption{Training rounds versus test accuracy of LeNet \cite{lecun1999object} on the CIFAR-10 with various system parameters $(N,K,T)$.}
\label{fig:cifar10_ablation}
\vspace{-0.2cm}
\end{figure*}

Figure \ref{fig:CIFAR10_IID_N240_K12} and Figure \ref{fig:CIFAR10_IID_N120_K24} show the performance comparison with $(N,K)=(240,12)$ and $(N,K)=(120,24)$, respectively. 
Similar to Section \ref{sec:Experiments} and Appendix \ref{app:exp_mnist}, we can observe that \namespace schemes show comparable test accuracy to the random and weighted random selection schemes while the \namespace provide better multi-round privacy guarantee $T$, and the user partitioning scheme has the worst test accuracy as its average aggregation cardinality is much smaller than the other schemes. 
In particular, when $(N,K)=(120,24)$, the user partition scheme fails to train the model as the probability that all partitions are not available at each round becomes almost one. 

\section{Multi-round Privacy Analysis of the Conventional Random User Selection Strategies}
\label{app:random-selection}
In this appendix, we first theoretically study the multi-round privacy  of two random user selection strategies, and show that they have a very weak multi-round privacy of $T=1$ with high probability (for the case where $p_i=p, \forall i \in [N]$).  Furthermore, we also provide additional experiments showing that the server can reconstruct the local updates of all users with high accuracy when a random selection strategy is used. In the theoretical analysis, to simplify the problem, we assume that the model of the users have converged and don't change from one round to the next. However, in the experiments, we empirically evaluate the error in approximating the individual models of the users (via least-squares error estimation), and show that the server can approximate individual updates with very small error.

\subsection{Theoretical Analysis of the Random Selection Strategies}
We start by our theoretical results, where we consider the following two random selection schemes.
\begin{enumerate}[leftmargin=0.6cm]
    \item \textbf{$K$-uniform Random Selection}. In this scheme, at round $t$, $K$ users are selected uniformly at random from the set of available users $\mathcal U^{(t)}$ if $|\mathcal U^{(t)}| \geq K$. Otherwise, the server skips this round. 
    \item \textbf{I.I.D Random Selection}. In this scheme, at round $t$, each  user is selected with probability $ \frac{K}{N(1-p)}$ independently from the other available users, where $K < N(1-p)$. Hence, the expected number of selected users at each round is $K$ user.
\end{enumerate}
For both schemes, we show that the server can reconstruct all individual models after $N$ rounds in the worst-case scenario (assuming that the models do not change over $N$ rounds). Specifically, we show that the participation matrices in both schemes have full rank with high probability after $N$ rounds. This, in turn, implies that the server can reconstruct all local models after $N$ rounds with high probability in both schemes.  We provide our results formally next in Theorem \ref{thm:random-selection-multi-round}.

\begin{theorem}\label{thm:random-selection-multi-round} (Random selection schemes have a multi-round privacy guarantee $T=1$).
\begin{enumerate}[leftmargin=0.6cm]
    \item Consider the $K$-uniform random selection scheme, where $\min(K,N-K) \geq cN$. In this scheme, the server can reconstruct all individual models of the $N$ users after $N$ rounds with probability at least
    \begin{align}
        1-2e^{-c'N},
    \end{align}
    for some constant $c'>0$ that depends on $c$.
    
    \item Consider the i.i.d random selection scheme, where the users are selected according to Bern($\frac{K}{N(1-p)}$) distribution and let $t=K/N$. In this scheme, the server can reconstruct the individual models of the $N$ users after $N$ rounds with probability at least 
    \begin{align}
    1 - 2N(1-t)^N - (1+o_N(1)) N(N-1)(t^2 + (1-t)^2)^N,
    \end{align}
    which converges to $1$ exponentially fast if $t \in (0, 1/2)$ is a fixed constant. 
\end{enumerate}
\end{theorem}

\begin{proof}
We first note that if the participation matrix has full rank after $N$ rounds, then the server  can reconstruct the model of each individual user. Hence, we analyze the probability of the $N \times N$ participation matrix being full rank. We now consider each scheme separately. 
\begin{enumerate}[leftmargin=0.6cm]
    \item In the $K$-uniform random selection scheme, the probability that the participation matrix after $N$ rounds $\mathbf P^{(N)}$ has full rank is lower-bounded as follows \cite{tran2020smallest}, when $\min(K,N-K) \geq cN$, 
\begin{align*}
    \Pr[\mathbf P^{(N)} \ \text{has full rank}] \geq 1-2e^{-c'N},
\end{align*}
for some constant $c'>0$ that depends on $c$. Hence, it follows that the server can reconstruct all individual models with probability at least $1-2e^{-c'N}$.

\item In the i.i.d random selection scheme, the probability that the participation matrix after $N$ rounds $\mathbf P^{(N)}$ has full rank is lower-bounded as follows \cite{jain2020singularity}
\begin{align*}
 \Pr[\mathbf P^{(N)} \ \text{has full rank}] \geq  1 - 2N(1-t)^N - (1+o_N(1)) N(N-1)(t^2 + (1-t)^2)^N,
\end{align*}
which converges to $1$ exponentially fast if $t=K/N \in (0, 1/2)$ is a fixed constant. Hence, it follows that the probability the server can reconstruct all individual models is lower-bounded by the same probability.  
\end{enumerate}
\end{proof}

\begin{remark}\normalfont 
Our experimental results in Section \ref{sec:Experiments} also show that the multi-round privacy guarantee of the $K$-uniform random selection scheme goes to $1$ after almost $N$ rounds as shown in Fig. \ref{fig:priavcy}.
\end{remark}

\subsection{Experimental Results}
We now empirically evaluate the error in approximating the individual gradients of the users (via least-squares error estimation), and show that the server can approximate individual gradients of all users with a very small error when $K$-uniform random selection is used. To do so, we implement a reconstruction algorithm utilizing the least-squares method, and measure the $L_2$ distance between the true gradients and reconstructed gradients.
We consider a FL setting with $N=40$ users, where the server aims to choose $K=8$ users at every round, to train the LeNet in \cite{lecun1999object} on the CIFAR-10 dataset with Non-IID setting, which is the same as the setting in Appendix \ref{app:exp_mnist}.

Let $\mathbfsl{\delta}_i^{(t)}$ be the  gradient of user $i$ at round $t$, i.e., $\mathbfsl{\delta}_i^{(t)} = \mathbfsl{x}_i^{(t)} - \mathbfsl{x}^{(t)}$, 
and $\mathbfsl{\delta}^{(t)}$ be the global update at round $t$, i.e., $\mathbfsl{\delta}^{(t)} = \mathbfsl{x}^{(t+1)} - \mathbfsl{x}^{(t)} = {\mathbf{\Delta}^{(t)}}^{\top}_{\text{individual}} \mathbfsl{p}^{(t)}$ 
where $\mathbf{\Delta}^{(t)}_{\text{individual}} = \left[ w_1\mathbfsl{\delta}_1^{(t)}, \ldots,w_N\mathbfsl{\delta}_N^{(t)} \right]^{\top} \in \mathbb{R}^{N\times d}$.
After a sufficiently large number of rounds $t_0$, the global model at the server converges and does not change much across the rounds, which results in that local updates also do not change much across the rounds. Then, we have
\begin{equation}
    \mathbf{\Delta}^{(t_0;t_1)}_{\text{global}} = \mathbf{P}^{(t_0;t_1)}
    \mathbf{\Delta}^{(t_0)}_{\text{individual}} + \mathbf{Z},
\end{equation}
where $\mathbf{\Delta}^{(t_0;t_1)}_{\text{global}}$ denotes the concatenate of the global updates from round $t_0$ to round $t_1 - 1$, i.e., $\mathbf{\Delta}^{(t_0;t_1)}_{\text{global}} = \left[ \mathbfsl{\delta}^{(t_0)}, \ldots \mathbfsl{\delta}^{(t_1-1)} \right]^{\top}\in \mathbb{R}^{(t_1-t_0)\times d}$ for $t_1 > t_0$,
$\mathbf{P}^{(t_0;t_1)} \in \{0,1\}^{(t_1-t_0)\times N}$ is the participation matrix from round $t_0$ to round $t_1-1$,
and $\mathbf{Z}$ denotes the perturbation (or noise) incurred by the local updates across the rounds. 


The server can then estimate $\mathbf{\Delta}^{(t_0)}_{\text{individual}}$ by utilizing the least-squares method as follows
\begin{equation}\label{eq:recon_ind}
    \hat{\mathbf{\Delta}}^{(t_0)}_{\text{individual}} = \left( {\mathbf{P}^{(t_0;t_1)}}^\top \mathbf{P}^{(t_0;t_1)} \right)^{-1} {\mathbf{P}^{(t_0;t_1)}}^\top \mathbf{\Delta}^{(t_0;t_1)}_{\text{global}},
\end{equation}
and we measure the reconstruction error as follows 
\begin{equation}\label{eq:recon_error}
   e^{(t_0)}_i = \frac{\lVert \mathbfsl{\delta}_i^{(t_0)} - \hat{\mathbfsl{\delta}}_i^{(t_0)} \rVert_2^2}{\lVert \mathbfsl{\delta}_i^{(t_0)} \rVert_2^2},
\end{equation}
where $\hat{\mathbfsl{\delta}}_i^{(t_0)}$ denotes the reconstructed gradient of user $i$, which corresponds to $i$-th row of $\hat{\mathbf{\Delta}}^{(t_0)}_{\text{individual}}$ in \eqref{eq:recon_ind}.
On the other hand, in \namespace with multi-round privacy guarantee $T=2$, the server cannot estimate the individual gradients by utilizing \eqref{eq:recon_ind} because $\mathbf{P}^{(t_0;t_1)}$ is not full rank hence the inverse of ${\mathbf{P}^{(t_0;t_1)}}^\top \mathbf{P}^{(t_0;t_1)}$ does not exist. The best that the server can do is to estimate $\sum_{i\in\mathcal{G}_j}\mathbfsl{\delta}_i^{(t_0)}$, where $\mathcal{G}_j$ is the index set of the users in the $j$-th batch. The server can then estimate $\mathbfsl{\delta}_i^{(t_0)}$ by dividing the estimate of $\sum_{i\in\mathcal{G}_j}\mathbfsl{\delta}_i^{(t_0)}$ by $T$, where $i \in \mathcal{G}_j$.

\begin{figure*}[t!]%
    \centering
    \subfigure[$K(=8)$-uniform random selection ($T=1$).]{%
    \vspace{-0.2cm}
    \label{fig:histo_error_random}%
    \includegraphics[height=1.65in, trim=0.1cm 0.1cm 0.15cm 0.15cm, clip]{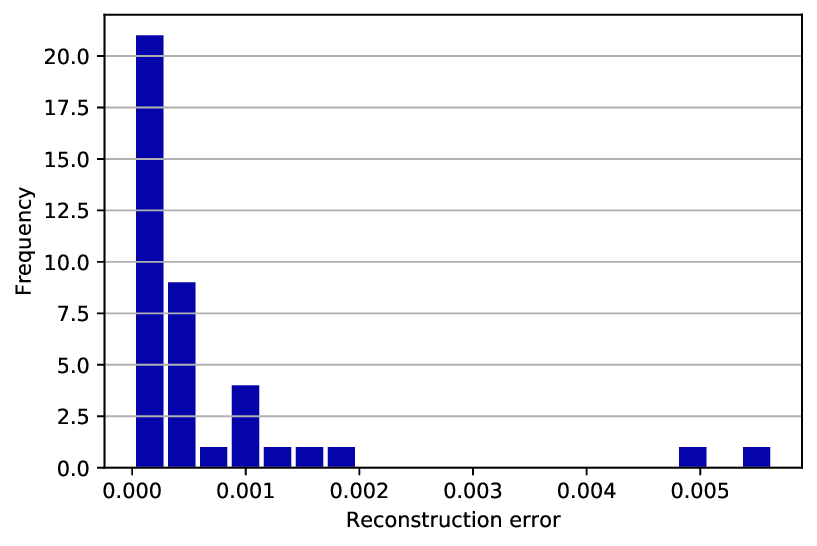}}%
    \;\;
    \subfigure[\namespace ($T=2$).]{%
    \vspace{-0.2cm}
    \label{fig:histo_error_proposed}%
    \includegraphics[height=1.65in, trim=0.1cm 0.1cm 0.15cm 0.15cm, clip]{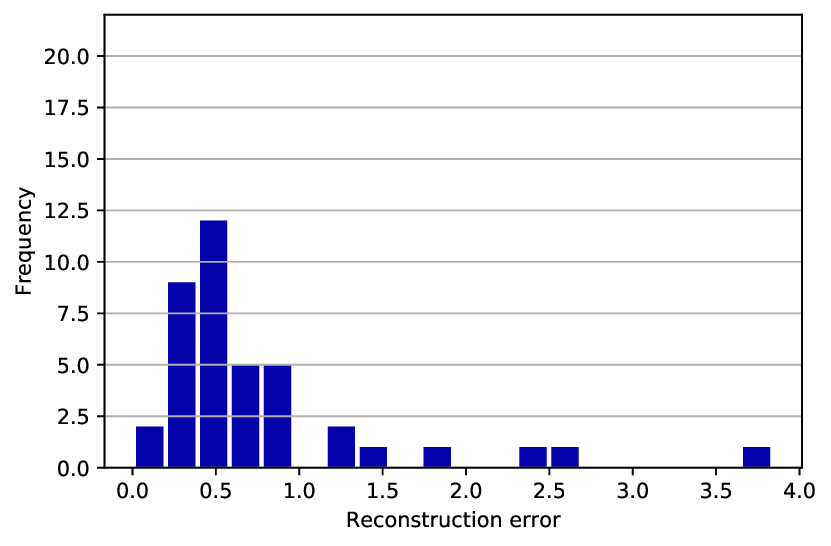}}%
    \vspace{-0.2cm}
    \caption{\footnotesize Histogram of the reconstruction error defined in \eqref{eq:recon_error} when the $K(=8)$-uniform random selection or \namespace ($T=2$) scheme is used to train the LeNet on the CIFAR-10 dataset. The average reconstruction errors of $K(=8)$-uniform random selection  and \namespace ($T=2$) are $6.715\times 10^{-3}$ and $0.7829$, respectively, which implies that the server can reconstruct all local updates when $K(=8)$-uniform random selection is used while the server cannot reconstruct the local updates when \namespace ($T=2$) is used.}
\label{fig:histo_error}
\vspace{-0.2cm}
\end{figure*}

\begin{figure}[t!]
    \centering
    \includegraphics[width=0.80\textwidth]{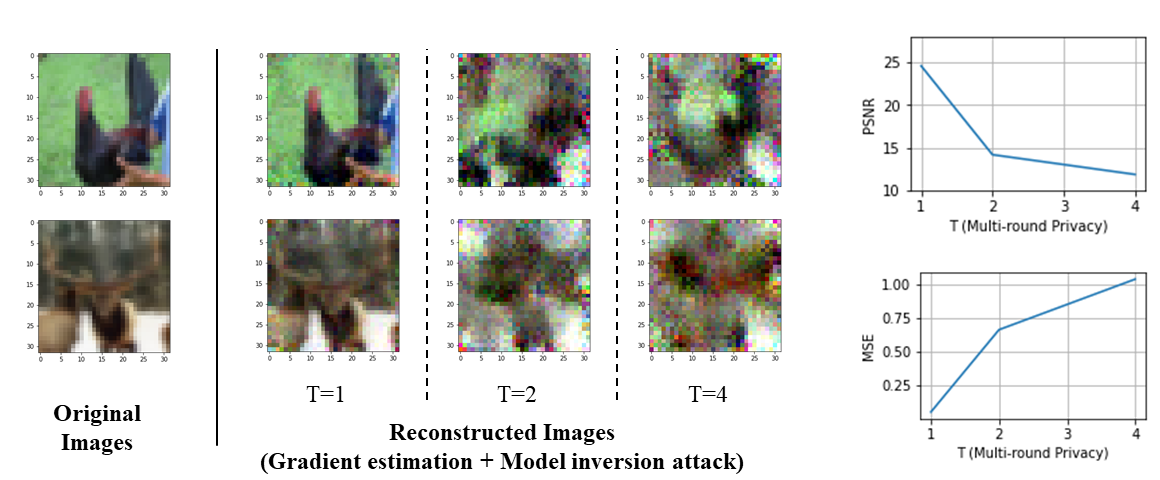}
    \caption{\footnotesize Comparison of the reconstructed images using the model inversion attack~\cite{geiping2020inverting} with different value of multi-round privacy guarantee $T$ (left) and measurement of similarity between the reconstructed images and the original images, where $\text{PSNR}=\infty$ and $\text{MSE}=0$ for two identical images (right).}
    \label{fig:recon_MSE_PSNR}
\end{figure}

Figure \ref{fig:histo_error_random} and Figure \ref{fig:histo_error_proposed} show the histogram of the reconstruction error of the individual gradients when the $K$-uniform random selection scheme and \namespace ($T=2$) scheme are used, respectively. 
We set $t_0=1460$ and $t_1=1500$ in this experiment.
We observe that the $K$-uniform random selection scheme has much smaller average reconstruction error $\frac{1}{N}\sum_{i=1}^{N}e^{(t_0)}_i=6.715\times 10^{-3}$ than the average reconstruction error of \namespace ($T=2$), which implies that the server can reconstruct all local gradients as the $K$-uniform random selection scheme has a multi-round privacy guarantee $T=1$.

Finally, the server can reconstruct the training images by applying model inversion attack~\cite{geiping2020inverting} to the reconstructed gradient $\hat{\mathbfsl{\delta}}_i^{(t_0)}$.
Figure~\ref{fig:recon_MSE_PSNR} the reconstructed images of random selection scheme ($T=1$) and \namespace ($T=2,4$). 
We measure the reconstruction performance using peak signal-to-noise ratio (PSNR) and mean square error (MSE).
Large PSNR and small MSE indicate more similarity between the reconstructed and original images, and hence we can observe that random selection scheme ($T=1$) leaks much more information about the original image than \namespace ($T=2,4$).

\end{document}